\documentclass[letterpaper,11pt]{article}

\usepackage[singlespacing]{setspace}
\usepackage[margin=1in]{geometry}

\usepackage[T1]{fontenc}
\usepackage[utf8]{inputenc}

\usepackage{natbib}

\usepackage[utf8]{inputenc} %
\usepackage[T1]{fontenc}    %
\usepackage{hyperref}       %
\hypersetup{
    colorlinks,
    breaklinks,
    linkcolor = blue,
    citecolor = blue,
    urlcolor  = blue,
}
\usepackage{url}            %
\usepackage{amsfonts}       %
\usepackage{nicefrac}       %
\usepackage{microtype}      %
\usepackage{amsmath,bm}
\usepackage{amsthm}
\usepackage[algo2e, ruled, vlined]{algorithm2e}
\usepackage{algorithmicx}
\usepackage{algorithm}
\usepackage[noend]{algcompatible}
\usepackage{bbm}      
\usepackage{mathtools}  \usepackage{matlab-prettifier}
\usepackage{amsopn}
\usepackage{amssymb}
\usepackage{graphicx}
\usepackage{xcolor}
\usepackage{footnote}
\usepackage{makecell}
\makesavenoteenv{tabular}
\makesavenoteenv{table}

\usepackage{booktabs}

\usepackage{xcolor}
\usepackage{soul}
\usepackage{changepage}

\definecolor{Green}{rgb}{0.13, 0.65, 0.3}
\definecolor{Amber}{rgb}{0.3, 0.5, 1.0}
\usepackage[]{color-edits}
\addauthor{HL}{Green}
\addauthor{TJ}{Amber}
\usepackage{comment}

\usepackage{minitoc}

\newcommand{\Alg}{\textbf{Alg}}

\sethlcolor{gray}

\newcommand{\mbD}{\mathbb{D}}

\newcommand{\Ind}[1]{ \field{I}{\left\{{#1}\right\}} }

\newcommand{\norm}[1]{\left\|{#1}\right\|}

\newcommand{\wtilC}{\widetilde{C}}

\newcommand{\naturalnum}{\mathbb{N}}

\newcommand{\bi}{\begin{itemize}}
\newcommand{\ei}{\end{itemize}}

\theoremstyle{theorem} 
\newtheorem{theorem}{Theorem}
\newtheorem{lemma}[theorem]{Lemma}

\newtheorem{corollary}[theorem]{Corollary}

\newtheorem{definition}[theorem]{Definition}
\newtheorem{remark}[theorem]{Remark}

\newtheorem{fact}[theorem]{Fact}
\newtheorem{example}[theorem]{Example}

\usepackage{amssymb}%
\usepackage{pifont}%

\allowdisplaybreaks

\DeclareMathOperator*{\argmin}{\arg\!\min}

\usepackage{amsmath}
\usepackage{amssymb}
\usepackage{graphicx}
\usepackage{mathtools}
\usepackage{enumerate}
\usepackage{enumitem}
\usepackage{footnote}
\usepackage{float}
\usepackage{xspace}
\usepackage{multirow}
\usepackage{nicefrac}
\usepackage{wrapfig}
\usepackage{framed}
\usepackage{url}
\usepackage{tikz}
\usetikzlibrary{arrows,chains,matrix,positioning,scopes}

\newcommand{\calS}{{\mathcal{S}}}

\newcommand{\calK}{{\mathcal{K}}}
\newcommand{\calE}{{\mathcal{E}}}

\newcommand{\calT}{{\mathcal{T}}}

\newcommand{\calF}{{\mathcal{F}}}
\newcommand{\calV}{{\mathcal{V}}}

\newcommand{\calG}{{\mathcal{G}}}

\newcommand{\Conv}{\texttt{Conv}}

\newcommand{\field}[1]{\mathbb{#1}}

\newcommand{\E}{\field{E}}

\renewcommand{\P}{\field{P}}

\newcommand{\Reg}{{\text{\rm Reg}}}

\newcommand{\order}{\ensuremath{\mathcal{O}}}
\newcommand{\otil}{\ensuremath{\widetilde{\mathcal{O}}}}

\newcommand{\TV}{\texttt{TV}}

\newcommand{\ENR}{\textsc{Enr}}
\newcommand{\NR}{\textsc{Nr}}
\newcommand{\ExtR}{\textsc{Reg}}

\newcommand{\rbr}[1]{\left(#1\right)}

\newcommand{\sbr}[1]{\left[#1\right]}

\newcommand{\cbr}[1]{\left\{#1\right\}}

\newcommand{\abr}[1]{\left|#1\right|}
\newcommand\numberthis{\addtocounter{equation}{1}\tag{\theequation}}

\DeclareFontFamily{OMX}{MnSymbolE}{}
\DeclareFontShape{OMX}{MnSymbolE}{m}{n}{
    <-6>  MnSymbolE5
   <6-7>  MnSymbolE6
   <7-8>  MnSymbolE7
   <8-9>  MnSymbolE8
   <9-10> MnSymbolE9
  <10-12> MnSymbolE10
  <12->   MnSymbolE12}{}
\DeclareSymbolFont{mnlargesymbols}{OMX}{MnSymbolE}{m}{n}
\SetSymbolFont{mnlargesymbols}{bold}{OMX}{MnSymbolE}{b}{n}
\DeclareMathDelimiter{\llangle}{\mathopen}{mnlargesymbols}{'164}{mnlargesymbols}{'164}
\DeclareMathDelimiter{\rrangle}{\mathclose}{mnlargesymbols}{'171}{mnlargesymbols}{'171}

\usepackage{prettyref}
\newcommand{\pref}[1]{\prettyref{#1}}

\newcommand{\savehyperref}[2]{\texorpdfstring{\hyperref[#1]{#2}}{#2}}
\newrefformat{eq}{\savehyperref{#1}{Eq.~\eqref{#1}}}
\newrefformat{eqn}{\savehyperref{#1}{Equation~\eqref{#1}}}
\newrefformat{lem}{\savehyperref{#1}{Lemma~\ref*{#1}}}
\newrefformat{def}{\savehyperref{#1}{Definition~\ref*{#1}}}
\newrefformat{line}{\savehyperref{#1}{line~\ref*{#1}}}
\newrefformat{thm}{\savehyperref{#1}{Theorem~\ref*{#1}}}
\newrefformat{corr}{\savehyperref{#1}{Corollary~\ref*{#1}}}
\newrefformat{cor}{\savehyperref{#1}{Corollary~\ref*{#1}}}
\newrefformat{col}{\savehyperref{#1}{Corollary~\ref*{#1}}}
\newrefformat{sec}{\savehyperref{#1}{Section~\ref*{#1}}}
\newrefformat{app}{\savehyperref{#1}{Appendix~\ref*{#1}}}
\newrefformat{assum}{\savehyperref{#1}{Assumption~\ref*{#1}}}
\newrefformat{ex}{\savehyperref{#1}{Example~\ref*{#1}}}
\newrefformat{fig}{\savehyperref{#1}{Figure~\ref*{#1}}}
\newrefformat{tab}{\savehyperref{#1}{Table~\ref*{#1}}}
\newrefformat{alg}{\savehyperref{#1}{Algorithm~\ref*{#1}}}
\newrefformat{rem}{\savehyperref{#1}{Remark~\ref*{#1}}}
\newrefformat{conj}{\savehyperref{#1}{Conjecture~\ref*{#1}}}
\newrefformat{prop}{\savehyperref{#1}{Proposition~\ref*{#1}}}
\newrefformat{proto}{\savehyperref{#1}{Protocol~\ref*{#1}}}
\newrefformat{prob}{\savehyperref{#1}{Problem~\ref*{#1}}}
\newrefformat{claim}{\savehyperref{#1}{Claim~\ref*{#1}}}
\newrefformat{que}{\savehyperref{#1}{Question~\ref*{#1}}}
\newrefformat{obs}{\savehyperref{#1}{Observation~\ref*{#1}}}

\title{Online Learning for Uninformed Markov Games: \\ Empirical Nash-Value Regret and Non-Stationarity Adaptation}

\author{%
Junyan Liu\thanks{University of Washington. Email: \texttt{junyanl1@cs.washington.edu}.}
\and
Haipeng Luo\thanks{University of Southern California. Email: \texttt{haipengl@usc.edu}.}
\and
Zihan Zhang\thanks{Hong Kong University of Science and Technology. Email: \texttt{zihanz@ust.hk}.}
\and
Lillian J. Ratliff\thanks{University of Washington. Email: \texttt{ratliffl@uw.edu}.
}
}

\date{}

\begin{document}

\doparttoc
\faketableofcontents %

\maketitle

\begin{abstract}
We study online learning in two-player uninformed Markov games, where the opponent’s actions and policies are unobserved. In this setting, \citet{tian2021online} show that achieving no-external-regret is impossible without incurring an exponential dependence on the episode length $H$. 
They then turn to the weaker notion of Nash-value regret and propose a V-learning algorithm with regret $\otil(K^{2/3})$ after $K$ episodes.
However, their algorithm and guarantee do not adapt to the difficulty of the problem:
even in the case where the opponent follows a fixed policy and thus $\otil(\sqrt{K})$ external regret is well-known to be achievable,
their result is still the \textit{worse} rate $\otil(K^{2/3})$ on a \textit{weaker} metric.

In this work, we fully address both limitations.
First, we introduce \textit{empirical Nash-value regret}, 
a new regret notion that is strictly stronger than Nash-value regret and naturally reduces to external regret when the opponent follows a fixed policy.
Moreover, under this new metric, we propose a parameter-free algorithm that achieves an $\otil \big(\min\{\sqrt{K} + (CK)^{1/3}, \sqrt{LK}\}\big)$ regret bound, where $C$ quantifies the ``variance'' of the opponent’s policies and $L$ denotes the number of policy switches (both at most $\order(K)$). 
Therefore, our results not only recover the two extremes---$\otil(\sqrt{K})$ external regret when the opponent is fixed and $\otil(K^{2/3})$ Nash-value regret in the worst case---but also smoothly interpolate between these extremes by automatically adapting to the opponent's non-stationarity.
We achieve so by first providing a new analysis of the epoch-based V-learning algorithm by \cite{mao2022improving}, establishing an $\otil(\eta C + \sqrt{K/\eta})$ regret bound, where $\eta$ is the epoch incremental factor.
Next, we show how to adaptively restart this algorithm with an appropriate $\eta$ in response to the potential non-stationarity of the opponent, eventually achieving our final results.
\end{abstract}

\section{Introduction}

Multi-agent reinforcement learning (MARL), often modeled as a Markov game (MG), provides a general framework for studying sequential decision-making problems involving multiple strategic agents whose actions jointly influence a shared environment.
Recent advances have demonstrated the empirical success of MARL in domains with complex strategic interactions, including the game of Go \citep{silver2016mastering,silver2017mastering}, Poker \citep{brown2019superhuman}, large-scale video games \citep{vinyals2019grandmaster}, and autonomous driving \citep{shalev2016safe,zhou2021smarts}.

A substantial body of prior work studies MGs in the self-play setting, where all players follow the same learning algorithm to optimize their joint behavior, typically with the goal of minimizing the number of episodes required to identify a good joint policy \citep{bai2020provable,bai2020near,liu2021sharp,mao2022improving,jin2024v}. While this formulation has led to significant algorithmic progress, it abstracts away challenges that arise in many practical multi-agent systems. In such settings, an agent may repeatedly interact with opponents whose learning rules, objectives, or update schedules are unknown and not aligned with its own. This motivates the study of learning Markov games with \emph{arbitrary opponents} \citep{xie2020learning,tian2021online}, where an agent must balance exploration of unknown environments with strategic decision-making against opponents that may follow arbitrary, potentially history-dependent policies.

Depending on the observation model, MGs with arbitrary opponents can be divided into the informed setting, where the opponent’s actions are observable, and the uninformed setting, where the opponent’s actions are unobserved. While near-optimal regret bounds have been established for the informed setting \citep{xie2020learning,tian2021online}, 
learning in uninformed MGs is significantly more challenging, since the lack of access to the opponent’s actions prevents explicitly learning the transition model or maintaining a table of state-action value ($Q$-value table).
Indeed, \citet{tian2021online} show that an exponential dependence on the episode length $H$ is unavoidable when performance is measured by the standard external regret.
As a remedy, \citet{tian2021online} consider a weaker performance measure that compares the algorithm's total reward to the Nash value of the MG, known as Nash-value regret.
They then show that a variant of the V-learning algorithm \citep{bai2020near} with an appropriate choice of parameters achieves an $\otil(K^{2/3})$ Nash-value regret bound after $K$ episodes.
This remains the best result in this setting, with the current best lower bound being $\Omega(\sqrt{K})$.

However, while the tightness of their result in the worst case remains unknown, it is certainly suboptimal in special cases. 
For example, when the opponent follows a fixed policy, then the learner is simply facing a fixed environment, in which case $\Theta(\sqrt{K})$ \textit{external} regret is well-known~\citep{azar2017minimax};
on the other hand, the result of \citet{tian2021online} for this case not only is about the weaker Nash-value regret (see \pref{rem:discussion_reg}), but also has a worse rate of $\otil(K^{2/3})$.
\citet{tian2021online} leave it as an open question whether this limitation is fundamental or merely an artifact of the analysis,
but more generally, this begs the following natural questions:
\begin{center}
\emph{
1. Is there a better regret notion that naturally interpolates between standard external regret and Nash-value regret as the opponent's non-stationarity increases?}\\
\emph{
2. Correspondingly, for such adaptive regret notions, are there efficient algorithms whose regret  smoothly interpolates between $\otil(\sqrt{K})$ and $\otil(K^{2/3})$ as the opponent's non-stationarity increases?}
\end{center}

\paragraph{Contributions.}
In this paper, we provide affirmative answers to both questions. %
Specifically, our main contributions are summarized as follows.
\begin{itemize}[leftmargin=*]
    \item For two-player uninformed MGs, we introduce a new regret notion termed \emph{empirical Nash-value regret} (\ENR{}), which is the Nash value for a game where at each state the opponent is restricted to the policies that they have played over the $K$ episodes.
    Consequently, \ENR{} not only is strictly stronger than the Nash-value regret (\NR{}) considered in prior work,
    but also reduces to the standard external regret 
    when the opponent follows a fixed policy. %

    \item Towards achieving our second goal, we start by introducing a novel analysis of the epoch-based V-learning algorithm of \citet{mao2022improving}, which turns out to be more manageable than the original V-learning algorithm~\citep{bai2020near, tian2021online}.
    Note that \citet{mao2022improving} analyze this algorithm under the self-play setting, which is very different from our analysis that deals with an arbitrary opponent.
    Specifically, we establish an $\otil(\eta C+\sqrt{K/\eta})$ regret bound under the new metric \ENR{}, where $\eta\in(0,1/H]$ is the epoch incremental factor and $C$ is a certain non-stationarity measure that quantifies the ``variance'' of the opponent’s policies (see \pref{eq:def_C_epoch_Vol} for its formal definition). When the opponent uses a single fixed policy, we have $C=0$, and thus choosing $\eta=1/H$ yields an $\otil(\sqrt{K})$ external regret bound. This result is the first to show that V-learning–type algorithms can achieve an $\otil(\sqrt{K})$ external regret bound in a stationary environment, answering the open question raised by \citet{tian2021online}. 
    More generally, setting $\eta = \min\{1/H, K^{1/3}C^{-2/3}\}$ leads to a bound of order $\otil(\sqrt{K}+(CK)^{1/3})$.
    The issue is of course that this tuning requires the knowledge of $C$.
    
    \item To address the limitation that epoch-based V-learning cannot automatically adapt to unknown non-stationarity, we further propose a meta-algorithm that repeatedly restarts epoch-based V-learning in response to potential non-stationarity in the environment. In addition, motivated by the observation that the environment is nearly stationary when the opponent switches policies infrequently (yet $C$ can be as large as $\Omega(T)$ in this case), we further equip the meta-algorithm with a mechanism to detect such changes. As a result, the final (parameter-free) algorithm achieves a regret bound of $\otil(\min\{\sqrt{K}+(CK)^{1/3},\sqrt{LK}\})$ for \ENR{}, where $L$ denotes the total number of policy switches by the opponent. 
    This shows that our algorithm automatically adapts to both non-stationarity measures ($C$ and $L$) and enjoys a bound on $\ENR{}$ that smoothly interpolates between $\order(\sqrt{K})$ and $\otil(K^{2/3})$ as $C$ and $L$ increase, completely resolving our second question.
\end{itemize}

\paragraph{Related work.} Markov games, also known as stochastic games \citep{shapley1953stochastic}, are a fundamental framework in MARL. Early work primarily established asymptotic convergence to Nash equilibrium under the assumption that the transition dynamics and rewards are known \citep{littman2001friend,hu2003nash,hansen2013strategy}. More recently, a growing literature has developed non-asymptotic guarantees for MGs without additional structural assumptions \citep{wei2017online,sidford2020solving,bai2020provable,xie2020learning,liu2021sharp,tian2021online,mao2022improving}. In the self-play setting, where all players deploy the same (or symmetric) learning algorithms, \citet{sidford2020solving,bai2020provable,liu2021sharp,mao2022improving} study the sample complexity of computing an $\epsilon$-approximate Nash equilibrium.

Our work is most closely related to \citet{wei2017online,xie2020learning,tian2021online}, which study Markov games with arbitrary opponents. In particular, \citet{wei2017online,xie2020learning} consider settings in which the opponent’s actions are observable, and propose algorithms that achieve $\sqrt{K}$-order regret bounds. In contrast, \citet{tian2021online} study the more challenging uninformed setting, where the opponent’s actions are unobserved. They show that external regret admits a lower bound of $\Omega\big(\min{\sqrt{2^{H}K},K}\big)$, which motivates the study of a weaker regret notion, \NR{}. Building on \citet{bai2020near}, they propose a V-learning–type algorithm that achieves $\otil(K^{2/3})$ regret under \NR{}.
Given the impossibility result for external regret in the uninformed setting, \citet{liu2022learning} investigate what additional assumptions are necessary to recover $\sqrt{K}$-order external regret bounds. 

In the single-agent RL literature, a line of work also studies how to adapt to the non-stationarity of a sequence of changing Markov Decision Processes (MDPs).
For example, \citet{wei2021non} studied the strong notion of dynamic regret (that competes to the best policy of each MDP) and proposed a black-box reduction approach to adapt to unknown non-stationarity.
\citet{jin2023no} focused on external regret and developed an algorithm that adapts to the unknown non-stationarity of the transitions (while allowing rewards to be arbitrary).
Since playing a Markov game with an arbitrary opponent can be viewed as a single agent interacting with a sequence of changing MDPs, one can apply those algorithms to our problem. 
However, unlike our results, the resulting bounds are all necessarily linear in $K$ in the worst case since the regret notions they consider are stronger.

Finally, we point out that a recent work by
\citet{appel2025regret} obtains a $\otil(\sqrt{K})$ 
guarantee on $\NR{}$ for \textit{turn-based} Markov games,
which we emphasize does \textit{not} translate to a same bound for our problem.
This is because in their definition of $\NR{}$, the Nash value is with respect to a game where at each state the min-player can make a decision based on the max-player's current action.
This makes their Nash value smaller than ours and thus their regret measure weaker than those considered by \citet{tian2021online} and our work.
To the best of our knowledge, whether the worst-case $\otil(K^{2/3})$ rate for our setting can be improved remains open.

\section{Preliminaries}

We consider an episodic two-player Markov game (MG) with finite state and action spaces.\footnote{Here, we focus on two players only for simplicity. 
From the learner's perspective, an $m$-player general-sum MG can be written as an equivalent two-player MG by grouping the other $m-1$ players into a single opponent with a joint action. See \citet[Section 5]{tian2021online} for more details.} 
Let $\Delta(X)$ denote the set of probability distributions over a finite set $X$. Such an MG is specified by a tuple $(\mathcal S, \mathcal A, \mathcal B, P, r, H)$, where
\begin{itemize}%
    \item $H$ is the number of steps in an episode;
    \item $S= \cup_{h \in [H+1]} S_h$ is the state space, where there is a single terminal state in $S_{H+1}$;
    \item $A = \cup_{h \in [H]} A_h$ and $B = \cup_{h \in [H]} B_h$ are the action space for the max-player (learner) and the min-player (opponent) respectively;
    \item $P$ is a collection of unknown transition functions $\cbr{P_h: S_h \times A_h \times B_h \to \Delta(S_{h+1}) }_{h \in [H]}$.
    \item $r$ is the collection of reward functions $\cbr{r_h: S_h \times A_h \times B_h \to [0,1] }_{h \in [H]}$.
\end{itemize}

In each episode $k \in [K]$, the Markov game starts from an adversarially chosen initial state $s_1^k \in \mathcal S_1$. At the beginning of this episode, each player commits to a policy that may depend on the entire history of the previous $k-1$ episodes. We denote our policy by $\mu^k = \{\mu_h^k\}_{h \in [H]}$ and the opponent’s policy by $\nu^k = \{\nu_h^k\}_{h \in [H]}$, where $\mu_h^k : \mathcal S_h \to \Delta(\mathcal A_h)$ and $\nu_h^k : \mathcal S_h \to \Delta(\mathcal B_h)$ are policies for step $h$. Neither player is allowed to modify its policy within an episode.
At each step $h \in [H]$, both players observe the current state $s_h^k$ and simultaneously select actions $a_h^k \in \mathcal A_h$ drawn from $\mu_h^k(s_h^k)$ and $b_h^k \in \mathcal B_h$ drawn from $\nu_h^k(s_h^k)$, respectively. They then observe reward $r_h^k:=r_h(s_h^k,a_h^k,b_h^k)$, after which the environment transitions to the next state according to $s^k_{h+1} \sim P_h(\cdot |s_h^k,a_h^k,b_h^k)$.
Importantly, at each step $h \in [H]$, the opponent’s action $b_h^k$ and policy $\nu_h^k$ are \emph{unobservable} to the learner. 
On the other hand, the opponent can observe the learner's actions and is allowed to choose policies arbitrarily and adaptively as a function of the interaction history, knowing the MG and learner's algorithm (but not their randomness) ahead of time.
However, for ease of exposition, we sometimes restrict our attention to an oblivious opponent whose policies cannot depend on the learner's past actions, and defer the details for a fully adaptive opponent to the appendix.

For a policy pair $(\mu,\nu)$, step $h\in [H]$, state $s \in S_h$, and actions $a \in A_h,b \in B_h$, we define the standard state value function and $Q$-value function as follows:
\begin{align*}
    V^{\mu,\nu}_h(s) &= \E_{\mu,\nu} \sbr{ \sum_{h'=h}^H r_{h'}(s_{h'},a_{h'},b_{h'}) \mid s_h=s }, \\
    Q^{\mu,\nu}_h(s,a,b) &= \E_{\mu,\nu} \sbr{ \sum_{h'=h}^H r_{h'}(s_{h'},a_{h'},b_{h'}) \mid s_h=s ,a_h=a,b_h=b},
\end{align*}
where $\{(s_{h'},a_{h'},b_{h'})\}_{h'\geq h}$ is the random trajectory generated by following $\mu$ and $\nu$ starting from $s_h=s$ for $V^{\mu,\nu}_h(s)$ or $s_h=s, a_h=a, b_h=b$ for $Q^{\mu,\nu}_h(s,a,b)$.

For shorthand, we define operators
\begin{align*}
    P_hV(s,a,b) = \E_{s' \sim P_h(\cdot|s,a,b)}[V(s')], \quad \text{and}\quad\mbD_{\mu,\nu}[Q] (s) = \E_{a \sim \mu(\cdot |s) ,b \sim \nu (\cdot|s)} \sbr{Q(s,a,b)}.
\end{align*}

Then, we have $V^{\mu,\nu}_h(s) =\mbD_{\mu_h,\nu_h}[Q^{\mu,\nu}](s)$ and $Q^{\mu,\nu}_h(s,a,b) = (r_h +P_hV_{h+1}^{\mu,\nu})(s,a,b)$.
Central to this new regret notion is the following definition of \textit{empirical state Nash values}:
for each $(h,s) \in [H] \times S_h$, recursively define
\begin{equation} \label{eq:enr_benchmark_def}
V_{h}^*(s) = \max_{\mu \in \Delta(A_h)} \min_{\nu \in \cbr{\nu^k_h(s) }_{k \in [K]} } \mbD_{\mu,\nu} \sbr{r_h +P_hV^*_{h+1} }(s),
\end{equation}
with $V_{H+1}^*(s)=0$ for all $s \in S_{H+1}$. 
It is ``empirical'' since in this definition we restrict the min-player to choose from the $K$ policies that are used by the opponent over the $K$ episodes.
If we were to relax this restriction and allow the min-player to choose any $\nu \in \Delta(B_h)$, then $V_{h}^*(s)$ became the exact and standard state Nash value.
With this definition, our proposed \textit{empirical Nash-value regret} ($\ENR$) is defined as
\begin{equation}\label{eq:ENR}
    \ENR_K = \sum_{k=1}^K \rbr{V^{*}_1(s_1^k) -V^{\mu^k,\nu^k}_1(s_1^k)},
\end{equation}
which compares the cumulative empirical Nash-values (at initial states) and the total expected reward of the learner over $K$ episodes.

In contrast, the Nash-value regret (denoted as $\NR{}_K$) considered by~\citet{tian2021online} replaces the empirical Nash values in \pref{eq:ENR} by the actual Nash values mentioned earlier.
On the other hand, the standard external regret is defined as $\ExtR_K = \max_{\mu} \sum_{k=1}^K \rbr{V^{\mu, \nu^k}_1(s_1^k) -V^{\mu^k,\nu^k}_1(s_1^k)}$,
comparing the total reward achieved by the best fixed policy and that of the learner.
From these definitions, the following fact is straightforward.

\begin{fact}
Empirical Nash-value regret is always stronger than Nash-value regret:
$\ENR_K \geq \NR_K$. Moreover, when the opponent uses a fixed policy, that is, $\nu_1=\cdots=\nu_K$, empirical Nash-value regret recovers external regret: $\ENR_K = \ExtR_K$.
\end{fact}

Consequently, the subsequent bounds we derived for $\ENR_K$ are also upper bounds on $\NR_K$, and they also become a bound on $\ExtR_K$ when the opponent uses a fixed policy.
More generally, $\ENR_K$ decreases as the diversity of the opponent's polices increases.

\textbf{Other notations.} For points $x_1,\ldots,x_n$, we use $\Conv(x_1,\ldots,x_n)$ to denote their convex hull. For two probability distributions $P,Q$ on the same measurable space, $\TV(P,Q)=\sup_{A} |P(A)-Q(A)|$ stands for the total variation distance. Define $\log^+(x)=\max\{\log(x),0\}$.
We also use $\calF_{h}^k$ to denote the history before step $h$ in episode $k$.

\section{Base Algorithm: Epoch V-learning and Analysis}
\label{sec:epoch_Vol}

As mentioned, our starting point is the epoch V-learning algorithm proposed by \citet{mao2022improving},
which simplifies the algorithmic design of the original V-learning methods (e.g., \citealp{bai2020near,jin2024v}). 
It turns out that it also allows easier analysis for $\ENR_K$ as we show below.
To introduce our novel analysis, 
we first review their algorithm (whose pseudocode is provided in \pref{alg:epoch_V_ol})
and point out a slight change that is critical for our purpose.

\setcounter{AlgoLine}{0}
\begin{algorithm}[t]
\DontPrintSemicolon
\caption{Epoch V-learning}
\label{alg:epoch_V_ol}
\textbf{Input}: total episodes $K$, confidence $\delta \in (0,1)$, epoch incremental factor $\eta>0$.

\textbf{Initialize}: $V^k_{H+1}\gets 0$ for all $k$; $V^1_h(s)\gets H-h+1,N_0(h,s)\gets 1,N_1(h,s)\gets 0, \calE(h,s) \gets \{1\},\calK_{1}(h,s)\gets\emptyset$, bandit subroutine $\textsc{AdvBandit}_{h,s}$ for all $h,s$; $\pi_h^1(a|s)\gets 1/|A_h|$ for all $s,a,h$.

\For{episode $k=1,2,\ldots,K$}{

\For{step $h=1,2,\ldots,H$}{

Receive state $s_h^k$. Set $\tau \gets |\calE(h,s_h^k)|$ to be the current epoch index for $(h,s_h^k)$ pair.

Take action $a^k_h \sim \pi_h^k(\cdot|s_h^k)$; observe reward $r_h^k$ and next state $s_{h+1}^k$.

Update $\calK_{\tau}(h,s_{h}^k) \gets \calK_{\tau}(h,s_{h}^k)  \cup \{k\}$ and
$N_{\tau}(h,s_{h}^k) \gets |\calK_{\tau}(h,s_{h}^k)|$.

\If(\hfill $\triangleright$ entering a new epoch){$N_{\tau}(h,s_{h}^k)=\lceil   (1+\eta)N_{\tau-1}(h,s_{h}^k) \rceil$}{
Update $\calE(h,s_h^k) \gets \calE(h,s_h^k) \cup \{\tau+1\}$.

Update $ V_h^{k+1}(s_h^k) \gets L_{\tau}(h,s_h^k)$ where $L_{\tau}(h,s)$ is defined in \pref{eq:non_increasing_update}.

Set $N_{\tau+1}(h,s_h^k) \gets 0$, $\calK_{\tau+1} (h,s_h^k) \gets \emptyset$, and $\pi^{k+1}_h(\cdot|s_h^k) \gets 1/|A_h|$.

Restart $\textsc{AdvBandit}_{h,s}$.

}\Else{
Set $V_h^{k+1}(s_h^k)\gets V_h^{k}(s_h^k)$.

Update $\pi_h^{k+1}(\cdot|s^k_h) \gets \textsc{AdvBandit}_{h,s} \big(  a_h^k,\frac{1}{H}(r_h^k+V^{k}_{h+1}(s_{h+1}^k)) \big) $. 
}

Set $V_h^{k+1}(s)\gets V_h^{k}(s)$ and
$\pi^{k+1}_h(\cdot|s) \gets \pi^{k}_h(\cdot|s)$ for all $s \in S_h \backslash s_h^k$.

}

}    
\end{algorithm}

\paragraph{Epoch schedule.}
At a high level, for each step-state pair $(h,s) \in [H] \times S_h$, \pref{alg:epoch_V_ol} partitions the set of all episodes where this pair is visited into epochs of geometrically increasing lengths,
and within each epoch, a new adversarial bandit subroutine is used to update the policy for this $(h,s)$ pair.
More concretely, $\calE(h,s)$ records the set of epoch indices for $(h,s)$ pair, $\calK_{\tau}(h,s)$ records the set of episodes in which $(h,s)$ is visited during epoch $\tau$, 
and $N_{\tau}(h,s)=|\calK_{\tau}(h,s)|$ is the visitation count during that epoch.
Given an epoch incremental factor $\eta >0$, if $N_{\tau}(h,s) =\lceil  (1+\eta)N_{\tau-1}(h,s) \rceil$, 
that is, the visitation count for $(h,s)$ pair increases by a factor of $1+\eta$,
then a new epoch for this pair is created.
We point out that~\citet{mao2022improving} always fix $\eta$ to be $1/H$,
but it is important for us to dynamically tune $\eta$ on the fly.

\paragraph{Optimistic Nash-value estimate.}
At the start of an epoch for $(h,s)$, a new adversarial bandit subroutine is initialized.
During the epoch, it is updated whenever state $s_h^k=s$ is visited, using the observed reward $r_h^k$ plus a \textit{bonus} 
$V^k_{h+1}(s_{h+1}^k)$, which, as we show in the analysis, serves as an optimistic estimate of the empirical Nash value $V^*_{h+1}(s_{h+1}^k)$.
Note that for each $(h,s)$, its optimistic estimate $V^k_h(s)$ stays the same within an epoch and is updated only when an epoch for $(h,s)$ ends.
Concretely, if the $\tau$-th epoch for pair $(h,s_h^k)$ ends at episode $k$, then we update $V^{k+1}_h(s_h^k) \gets  L_{\tau}(h,s_h^k)$ where 
$L_\tau(h,s)$ is the sum of the average ``reward plus bonus'' for state $s$ during the current epoch and a confidence width $\beta_{N_{\tau}(h,s)}$ term, truncated to $H-h+1$:
\begin{align} \label{eq:non_increasing_update}
   L_{\tau}(h,s) = \min \cbr{H-h+1,\frac{\sum_{j \in \calK_{\tau}(h,s)  } \rbr{r_h^j +V^{j}_{h+1}(s_{h+1}^j)}}{ N_{\tau}(h,s)}   +\beta_{N_{\tau}(h,s)}},
\end{align}
for all $\tau \geq 1$ (and $L_{0}(h,s)=H-h+1$ for all $h,s$).
Here, for any $n \in \mathbb{Z}_+$ and some known logarithmic factors $\iota>0$ (that depends on the choice of adversarial bandit subroutine), 
$\beta_n$ is a confidence width defined as
\begin{align} \label{eq:def_iota}
   \beta_n =    \sqrt{ \iota/n }\quad  \text{where}\quad \iota =\text{poly}(H,|A|,\log(K|S||A|/\delta)).
\end{align}

\paragraph{Adversarial bandit subroutine.} The adversarial bandit subroutine for step-state pair $(h,s)$ is denoted by $\textsc{AdvBandit}_{h,s}$. If $(h,s)$ is visited in episode $k$ during epoch $\tau$, $\textsc{AdvBandit}_{h,s}$ updates its model with the played action $a_h^k$ and a normalized value $\frac{1}{H}(r_h^k +V^{k}_{h+1}(s_{h+1}^k)) \in [0,1]$. Then, $\textsc{AdvBandit}_{h,s}$ outputs a policy $\pi^{k+1}_h(\cdot|s_h^k) \in \Delta(A_h)$. 
In this case, the external regret of $\textsc{AdvBandit}_{h,s} $ till the $n$-th visit of $(h,s)$ in epoch $\tau$ is defined as
\begin{equation} \label{eq:ext_reg_adv_bandit}
    \Reg^{\tau}_{h,s}(n) :=  \max_{\mu \in \Delta(A_h)} \frac{1}{H} \sum_{i=1 }^n  \rbr{\mbD_{\mu,\nu^{t_i}_h} \sbr{ r_h +P_h V_{h+1}^{t_i} }(s) -  \mbD_{\mu^{t_i}_h,\nu^{t_i}_h} \sbr{r_h +P_h V_{h+1}^{t_i} }(s)},
\end{equation}
where $t_i$ is the episode corresponding to the $i$-th visit of $(h,s)$ pair in epoch $\tau$.
As the visitation counter $N_{\tau}(h,s)$ may not reach $\lceil  (1+\eta)N_{\tau-1}(h,s) \rceil$ when \pref{alg:epoch_V_ol} ends, we require $\textsc{AdvBandit}_{h,s}$ to satisfy a high-probability anytime regret guarantee for some function $\xi$, polynomial in all arguments, and confidence parameter $\delta' \in (0,1)$:
\begin{equation} \label{eq:require_adv_bandit_reg}
\P \rbr{\forall n \in \naturalnum: \Reg^{\tau}_{h,s}(n) \leq    \xi \rbr{\log(1/\delta'),n,|A_h|}  } \geq 1-\delta'.
\end{equation}

Indeed, most existing adversarial bandit algorithms under the Follow-the-Regularized-Leader (FTRL) framework with implicit exploration (IX) estimators~\citep{neu2015explore} and doubling trick satisfy this requirement.
For completeness, we explicitly spell out the following example.

\begin{example}\citep[Theorem 2]{lecture13}\label{ex:adv_subroutine}
FTRL with $1/2$-Tsallis entropy, IX estimator, and doubling trick satisfies \pref{eq:require_adv_bandit_reg} with $\xi \rbr{\log(1/\delta'),n,|A_h|}=\order (\sqrt{|A_h|n \log(1/\delta')})$ and $\xi \rbr{\log(1/\delta'),n,|A_h|}=\order (\sqrt{|A_h|n \log(|A_h|/\delta')})$  for an oblivious environment and an adaptive environment, respectively. Here, an extra $|A_h|$ in the logarithmic term for the adaptive environment is due to an additional union bound. With such a bandit subroutine, we use $\iota=\Theta(H^2|A|\log(HK|A||S|/\delta))$ in \pref{alg:epoch_V_ol}.
\end{example}

\subsection{Main Results for Epoch V-Learning}
Now, we present a new result for epoch V-learning.
Note again that~\citep{mao2022improving} analyze this algorithm under the self-play setting, which is rather different from the analysis herein for the new regret metric $\ENR{}$ under an arbitrary opponent.
To this end, we introduce a non-stationarity measure capturing the ``variance'' of the opponent's policies, defined as
\begin{align} \label{eq:def_C_epoch_Vol}
  C=  \sum_{h=1}^H \sum_{k=1}^K \TV \rbr{\nu^k_h(s_h^k),\nu^*_h(s_h^k)  },
\end{align}
where $ \nu^{*}_{h}(s) \in \argmin_{  \nu \in \Conv \rbr{\nu^1_h(s),\ldots,\nu^K_h(s) }  }  \max_{\mu \in \Delta(A_h)} \mbD_{\mu ,\nu} \sbr{r_h +P_h V^*_{h+1} }(s)$
is the minimax policy for the opponent when restricted to play a mixture of empirical polices $\nu^1_h(s),\ldots,\nu^K_h(s)$.
In other words, $C$ measures the cumulative difference between the opponent's policy and a fixed minimax policy.
It is clear that in the worst case,  we have $C = \order(HK)$,
while in the best case when $\nu_h^k$ stays the same over all episodes, we have $C=0$.

We are now ready to present our main result.
For simplicity, we restrict our attention to oblivious opponents, but similar results hold for adaptive opponents too; see \pref{app:reg_bound_evol_adaptive}.

\begin{theorem} %
\label{thm:reg_bound_evol}
Suppose that the opponent is oblivious and $K \geq H|S|$. If we run \pref{alg:epoch_V_ol} with $\eta \in \sbr{|S|/K, 1/H}$ and the adversarial bandit subroutine is instantiated with \pref{ex:adv_subroutine} (so $\iota=\Theta(H^2|A|\log(HK|A||S|/\delta))$), 
then with probability at least $1-\delta$, the estimate holds:
\begin{align*}
\ENR_K \leq \sum_{k=1}^K \rbr{V^k_1(s_1^k)-V^{\mu^k,\nu^k}_1(s_1^k)} \leq \order \rbr{ \eta H^3C
+ H\sqrt{ \frac{\iota |S|K\log(K)}{\eta}  }   }.
\end{align*}
\end{theorem}

If $C$ is known, then setting $\eta=\min \big\{\frac{1}{H},\big( \frac{\iota |S|K}{H^4 C^2}\big)^{1/3}\big\}$ achieves $\ENR_K \leq \otil \big(\sqrt{H^3\iota |S| K  }  +  \rbr{ \iota H^5   |S|K   C }^{\frac{1}{3}} \big)$, smoothly interpolating between $\otil(\sqrt{K})$ and $\otil(K^{2/3})$.
We will address the issue that the knowledge of $C$ is required in achieving this bound in~\pref{sec:meta}.
Before that, we discuss the case when $C=0$.

\begin{corollary}[\textbf{External regret under a stationary opponent}]  \label{corr:external_reg_epoch_vol}
If $\nu^1=\cdots=\nu^{K}$, then running \pref{alg:epoch_V_ol} with $\eta=1/H$ and the adversarial bandit subroutine instantiated with \pref{ex:adv_subroutine} ensures that with probability at least $1-\delta$, $\ExtR_K=\ENR_K=\otil(\sqrt{H^5|A||S|K})$.
\end{corollary}

Note that this result also implies that one can run \pref{alg:epoch_V_ol} in the single-agent setting over a fixed MDP and obtain external regret bound $\ExtR_K=\otil(\sqrt{H^5|A||S|K})$,
since this setting can be modeled as a MG with a dummy second player with only one signle action at every state.
This matches the regret bound of the UCB-H algorithm \citep{jin2018q} and is close to the optimal bound $\otil(\sqrt{H^3 |A| |S| K})$~\citep{zhang2024settling}.
To the best of our knowledge, this is first result establishing a $\sqrt{K}$-type regret bound for V-learning algorithms over a fixed MDP, thereby resolving an open problem raised by~\citet{tian2021online}.

\subsection{Proof Sketch of \pref{thm:reg_bound_evol}}

In this subsection, we sketch the proof of \pref{thm:reg_bound_evol} and highlight the novelty of our analysis. %
In what follows, when we refer to a variable (e.g., $\calE(h,s),N_{\tau}(h,s)$, etc.), we refer to their final value after $K$ episodes.

Similar to \citep{tian2021online}, our analysis starts with a standard optimism argument, which shows $V^*_h(s) \leq V^k_h(s)$ for all $k,h,s$ (see \pref{lem:optimism}).
Then, we apply optimism and write the bound in terms of $\delta^k_h :=  (V^{k}_h -V^{\mu^k,\nu^k}_h)(s_h^k)$:
\begin{align*}
    \ENR_K \leq \sum_{k=1}^K \rbr{V^{k}_1(s_1^k) -V^{\mu^k,\nu^k}_1(s_1^k)} =\sum_{k=1}^K \delta_1^k =  \sum_{h=1}^H \sum_{k=1}^K \rbr{\delta_h^k-\delta_{h+1}^k}
\end{align*}
(with $\delta^k_{H+1}=0$ for all $k$). To prove the claimed bound, it suffices to show for each $h$,
\begin{equation} \label{eq:delta_hk_recursion_obj}
\sum_{k=1}^K \delta_{h}^k  \leq \sum_{k=1}^K \delta_{h+1}^k + \order \big( \eta H^2C
+ \sqrt{\iota |S|K\log(K)\eta^{-1}  }  \big).
\end{equation}

The rest of the proof for achieving this bound is novel as far as we know.
Specifically, for each $h,k$, we rewrite $\delta^k_h$ as
\begin{align*}
    \delta^k_h   =  \underbrace{\rbr{V^{k}_h(s_h^k) - \mbD_{\mu^k_h,\nu^k_h} \sbr{r_h+P_h V^{k}_{h+1} }(s_h^k)  }   }_{(1)_{h}^k}+  \underbrace{ \mbD_{\mu^k_h,\nu^k_h} \sbr{P_h (V^{k}_{h+1}-V^{\mu^k,\nu^k}_{h+1}) }(s_h^k) }_{(2)_{h}^k}.
\end{align*}

Since $\sum_{k=1}^K \delta_h^k = \sum_{k=1}^K (1)_h^k+\sum_{k=1}^K (2)_h^k $, we then bound each summation.

\paragraph{Bounding $\sum_k (2)_{h}^k$.} %
Note that $\sum_k \big( \mbD_{\mu^k_h,\nu^k_h} \big[P_h (V^{k}_{h+1}-V^{\mu^k,\nu^k}_{h+1})\big](s_h^k) - \delta^k_{h+1} \big) \leq \otil (H\sqrt{K})$ by Azuma-Hoeffding inequality for martingale difference sequence. Hence, for each $h \in [H]$, we have
\begin{equation} \label{eq:delta_hk_bound2}
    \sum_{k=1}^K (2)_{h}^k \leq \sum_{k=1}^K  \delta^k_{h+1} +\otil (H\sqrt{K}).
\end{equation}

\paragraph{Bounding $\sum_k (1)_{h}^k$.} Considering the objective in \pref{eq:delta_hk_recursion_obj} and the bound on $\sum_k (2)_{h}^k$ in \pref{eq:delta_hk_bound2}, we need only to show that $\sum_k (1)_{h}^k \leq \order \big( \eta H^2C
+ \sqrt{\iota |S|K\log(K)\eta^{-1}  }  \big)$.
Indeed, we have that
\begin{align*}
\sum_{k=1}^K (1)_h^k & =  \sum_{k=1}^K \rbr{V^{k}_h(s_h^k) - \mbD_{\mu^k_h,\nu^k_h} \sbr{r_h+P_h V^{k}_{h+1} }(s_h^k) } \\
& =  \sum_{s \in S_h} \sum_{\tau \in \calE(h,s)} \sum_{k \in K_{\tau}(h,s)} \rbr{V^{k}_h(s) - \mbD_{\mu^k_h,\nu^k_h} \sbr{r_h+P_h V^{k}_{h+1} }(s) } \\
&\leq  \sum_{s \in S_h} \sum_{\tau \in \calE(h,s)}N_{\tau}(h,s)  \rbr{ L_{\tau-1}(h,s) -L_{\tau}(h,s)   } + \order \Big(  \sum_{s \in S_h} \sum_{\tau \in \calE(h,s)} N_{\tau}(h,s) \beta_{N_{\tau}(h,s)} \Big),
\end{align*}
where in the last step, we use the fact $V^k_h(s) =L_{\tau-1}(h,s)$ for all $k \in \calK_{\tau}(h,s)$,
the definition of $L_{\tau}(h,s)$ from \pref{eq:non_increasing_update},
and a standard concentration argument applied to 
$\sum_k \mbD_{\mu^k_h,\nu^k_h} [r_h+P_h V^{k}_{h+1}](s_h^k) - \sum_{k \in \calK_{\tau}(h,s)} (r^k_h+V^k_{h+1}(s_{h+1}^k))$; see
\pref{eq:bound_one_two_terms} for more details. 
The upper bound of $\sum_{s \in S_h} \sum_{\tau \in \calE(h,s)} N_{\tau}(h,s) \beta_{N_{\tau}(h,s)} $ depends on the number of epochs for each $(h,s)$ pair, which is bounded according to the following lemma via a simple calculation.
\begin{lemma} \label{lem:bound_num_epoch}
For any $h \in [H]$ and $s \in S_h$, we have $|\calE(h,s)|\leq \left \lceil \frac{(1+\eta)\log(K)}{\eta} \right \rceil$.
\end{lemma}

Then, we can show that $  \sum_{s \in S_h} \sum_{\tau \in \calE(h,s)} N_{\tau}(h,s) \beta_{N_{\tau}(h,s)} \leq   \order \big( \sqrt{  \iota |S|K \log(K)  \eta^{-1} }   \big)$ by Cauchy–Schwarz inequality. It remains to bound $\sum_{s \in S_h} \sum_{\tau \in \calE(h,s)}N_{\tau}(h,s)  \rbr{ L_{\tau-1}(h,s) -L_{\tau}(h,s)   }$. For each $s \in S_h$, one can show that
\begin{align*}
&  \sum_{\tau \in \calE(h,s)}N_{\tau}(h,s)  \rbr{ L_{\tau-1}(h,s) -L_{\tau}(h,s)   }  \\
&=N_1(h,s)L_0(h,s) + \sum_{\tau=1}^{|\calE(h,s)|-1} L_{\tau}(h,s)  \rbr{ N_{\tau+1}(h,s)-N_{\tau}(h,s) } -N_{|\calE(h,s)|}(h,s)L_{|\calE(h,s)|}(h,s)   \\
&=  N_1(h,s) \rbr{L_0(h,s) -V^*_h(s)} +  \sum_{\tau=1}^{|\calE(h,s)|-1}\rbr{ N_{\tau+1}(h,s)-N_{\tau}(h,s) } \rbr{  L_{\tau}(h,s) - V^*_h(s)} \\
&\quad - \sum_{s \in S_h} N_{|\calE(h,s)|}(h,s) \rbr{ L_{|\calE(h,s)|}(h,s) -V^*_h(s)}  \\
&\leq  N_1(h,s) \rbr{L_0(h,s) - V^*_h(s)} +  \eta \sum_{\tau=1}^{|\calE(h,s)|-1} N_{\tau}(h,s) \rbr{  L_{\tau}(h,s) - V^*_h(s)} + |\calE(h,s)|H \\
&\leq \eta  \sum_{\tau \in \calE(h,s)}N_{\tau}(h,s) \rbr{  L_{\tau}(h,s) -V^*_h(s)} + \order \rbr{\frac{H \log(K)}{\eta}} ,  \numberthis{}  \label{eq:first_term_bound_main_text}
\end{align*}
where the first equality follows from the fact that for any $\{a_i\}_{i=1}^n,\{b_i\}_{i=1}^n$, we have $\sum_{i=1}^n a_i(b_{i-1}-b_i)=a_1b_0 + \sum_{i=1}^{n-1} (a_{i+1}-a_i)b_i -a_nb_n$, the first inequality uses facts  $N_{\tau+1}(h,s) -N_{\tau}(h,s) \leq \eta N_{\tau}(h,s) +1$ and $L_{\tau}(h,s)\geq V^*_h(s)$ by optimism, and the last inequality uses %
\pref{lem:bound_num_epoch} to bound $|\calE(h,s)| \leq \order (\log(K)/\eta )$ for all $(h,s)$.

Now, the following lemma is a key to handle $\sum_{s \in S_h} \sum_{\tau \in \calE(h,s)}N_{\tau}(h,s) \rbr{  L_{\tau}(h,s) -V^*_h(s)}$,
which is also the most technically novel part of our proof.
As discussed in \citep[Remark 4]{tian2021online}, the gap between the optimistic estimate $L_{\tau}(h,s)$ and $V^*_h(s)$ may not diminish. Indeed, the following lemma presents an upper bound that depends on the non-stationarity $C$ and the epoch incremental factor $\eta$.
\begin{lemma} \label{lem:gamma_bound}
Suppose that $\eta \in (0,1/H]$. With probability at least $1-\delta$, for any step $h \in [H]$, 
\begin{align*}
\sum_{s \in S_h} \sum_{\tau \in \calE(h,s)}N_{\tau}(h,s) \rbr{  L_{\tau}(h,s) -V^*_h(s)} \leq \order \rbr{ H^2 C
 +  H\sqrt{ \frac{\iota  |S| K\log(K)}{\eta}  } + \frac{H^2|S|\log(K)}{\eta}   } .
\end{align*}
\end{lemma}

The proof of \pref{lem:gamma_bound} relies on a recursive argument from step $h$ to $H$, which requires $\eta \leq 1/H$
to bypass an exponential dependence on $H$. We refer readers to \pref{app:supporting_lemma_epoch_vol} for details. 

Summing \pref{eq:first_term_bound_main_text} for all $s \in S_h$ and then applying \pref{lem:gamma_bound}, we obtain
\begin{align} \label{eq:bound_term_onehk_main_text}
   \sum_{k=1}^K (1)_h^k \leq \order \rbr{ \eta H^2 C 
+ H \sqrt{\eta \iota |S|  K\log(K) } +\sqrt{   \iota |S|K\log(K) \eta^{-1} }+ H |S| \log(K) \eta^{-1}  }.
\end{align}

The desired bound in \pref{eq:delta_hk_recursion_obj} is immediate according to \pref{eq:bound_term_onehk_main_text}, \pref{eq:delta_hk_bound2}. Finally, summing \pref{eq:delta_hk_recursion_obj} for all $h \in [H]$, using the fact $\eta \leq 1/H$ to bound $H^2  \sqrt{\eta \iota |S|  K\log(K) } \leq H\sqrt{  \iota |S|K\log(K) \eta^{-1} }$, and using $K\geq H|S|$ to bound $H^2|S|\log(K)/\eta \leq H\sqrt{  \iota |S|K\log(K) \eta^{-1} }$, we obtain the claimed regret bound.

\section{Adapting to Unknown Non-Stationarity: A Meta-Algorithm}\label{sec:meta}

As mentioned, \pref{alg:epoch_V_ol} cannot adapt to the unknown level of policy variance and requires tuning the epoch incremental factor $\eta$ based on $C$. 
In fact, even setting this issue aside, a regret bound solely depends on $C$ is not always satisfactory, because there are simple examples where the opponent is intuitively benign yet $C$ is still $\Omega(K)$---for example, when the opponent switches their policy only once.
In this section, we address both of these limitations simultaneously, leading to our final parameter-free algorithm that automatically adapts to both $C$ and another non-stationarity measure $L := 1+\sum_{k=1}^{K-1} \Ind{ \nu^k \neq \nu^{k+1} }$, that is, the number of policy switches by the opponent (plus one).

\subsection{Algorithm and Main Results}

\setcounter{AlgoLine}{0}
\begin{algorithm}[t]
\DontPrintSemicolon
\caption{Adaptive Epoch V-Learning}
\label{alg:block_alg}
\textbf{Input}: confidence $\delta \in (0,1)$, total episode $K$, absolute constant $c_0\geq 2$ specified in \pref{eq:def_event6}, epoch incremental scheduling $\{\eta^{(b)}_{\ell}\}_{b,\ell}$.

\For{block $b=1,2,\ldots$}{

\For{$\ell=1,2\ldots,2^{2b}+1$}{

Create a new instance of \pref{alg:epoch_V_ol}, denoted by $\Alg_{\ell}^{(b)}$, with input $\Big( K,\frac{\delta}{2K^6},\eta^{(b)}_{\ell} \Big)$.

Initialize $\Phi_{\ell}^{(b)} \gets 0$, $\calT_{\ell}^{(b)} \gets  \emptyset$, and $D \gets 0$.

\While{
$\Phi_{\ell}^{(b)} \leq D$
}{

Let $k$ be the current episode.

Run $\Alg_{\ell}^{(b)}$ to collect rewards $\{r_h^k\}_{h=1}^H$ and construct optimistic estimate $V^k_1(s_1^k)$.

Update $\calT_{\ell}^{(b)} \gets \calT_{\ell}^{(b)} \cup \{k\}$ and $\Phi_{\ell}^{(b)} \gets \sum_{k \in \calT_{\ell}^{(b)}} \rbr{V^k_1(s_1^k) -\sum_{h=1}^H r_h^k} +\sqrt{\iota \abr{\calT^{(b)}_{\ell}}  }$.

If $\ell\leq 2^{2b}$, set $D\gets 3c_0 H  \sqrt{\frac{ \iota |S||\calT_{\ell}^{(b)}|\log(K)}{\eta^{(b)}_{\ell} }}$; otherwise, set $D \gets 4c_0 H  \sqrt{\frac{ \iota |S|K\log(K)}{\eta^{(b)}_{\ell} }}$.

}

}
}    
\end{algorithm}

The proposed meta-algorithm, called Adaptive Epoch V-Learning, is shown in \pref{alg:block_alg}. 
It uses \pref{alg:epoch_V_ol} as a base algorithm and adaptively restarts it with different epoch incremental factors in response to potential environment changes. 
The meta-algorithm runs in blocks $b=1,2,\ldots$, and each block $b$ is further divided into a number of sub-blocks $\ell=1,2,\ldots,2^{2b}+1$. 
For each sub-block $\ell$ in block $b$, \pref{alg:block_alg} creates a new instance of \pref{alg:epoch_V_ol} with input $(K,\delta/(2K^6),\eta^{(b)}_{\ell})$. Here, if the opponent is oblivious, %
then one can set the epoch incremental factor $\eta^{(b)}_{\ell}$ as follows (for adaptive opponents, refer to \pref{thm:Bound_final_adaptive} for the choice of $\eta^{(b)}_{\ell}$):
\begin{equation}\label{eq:eta_b_ell}
    \eta^{(b)}_{\ell} = \begin{cases} 
  \frac{1}{H }, & \ell \leq 2^{2b} , \\
  \max \cbr{\frac{2^{-2b}}{H },\frac{|S|}{K}} ,& \ell  = 2^{2b}+1.
\end{cases}
\end{equation}
When running an instance of \pref{alg:epoch_V_ol}, we monitor a (computable) upper bound $\Phi_{\ell}^{(b)}$ on its $\ENR$,
and when it exceeds a threshold $D$ whose value follows the second term of the regret bound in \pref{thm:reg_bound_evol}, we terminate this sub-block.

\paragraph{High-level ideas.} 
If, for a moment, we only focus on the last sub-block for each block and assume $K$ is large so that $\max \cbr{\frac{2^{-2b}}{H },\frac{|S|}{K}} =\frac{2^{-2b}}{H }$.
Then the incremental factor schedule in \pref{eq:eta_b_ell} is simply performing a standard doubling trick on the unknown non-stationarity $C$, that is:
maintain a guess on $C$, set the incremental factor accordingly, and when $\Phi_{\ell}^{(b)}\leq D$, which we know cannot be true if the guess on $C$ is correct,
we restart the algorithm and double the guess.
This would be enough to obtain an $\otil \big(\sqrt{K} + (CK)^{1/3}\big)$ regret bound without knowing $C$.

To further adapt to $L$, the algorithm additionally introduces the first $2^{b}$ sub-blocks, each serving as a test for a potential policy switch. By \pref{corr:external_reg_epoch_vol}, if the opponent does not change its policy, then the condition $\Phi_{\ell}^{(b)}\leq D$ will hold and the algorithm keeps running. Therefore, the termination of a sub-block implies that at least one policy switch has occurred, leading to an $\otil(\sqrt{LK})$ regret bound.

Finally, since each of the first $2^{b}$ sub-blocks uses $\eta^{(b)}_{\ell}=1/H$ while the last sub-block uses $\eta^{(b)}_{\ell}=2^{-2b}/H$ (for a large $K$), when all $2^{b}$ sub-blocks terminate, the regret incurred by the first $2^{b}$ sub-blocks and that incurred by the last sub-block are of the same order. This justifies taking the minimum of the two bounds and yields the refined regret bound $\otil(\min\{\sqrt{K}+(CK)^{1/3},\sqrt{LK}\})$.

The final result is summarized in the following theorem, where when we refer to \pref{alg:epoch_V_ol}, we implicitly mean running it with \pref{ex:adv_subroutine} as the adversarial bandit subroutine, and hence $\iota$ is again of order $\Theta(H^2|A|\log(HK|A||S|/\delta))$.

\begin{theorem} \label{thm:Bound_final_oblivious}
For an oblivious opponent, running \pref{alg:block_alg} with the choice of $\eta^{(b)}_{\ell}$ specified in \pref{eq:eta_b_ell} guarantees, with probability at least $1-\delta$
\begin{align*}
\ENR_K \leq \otil \rbr{ H^2 |S| +     \min \cbr{  \sqrt{H^3\iota |S| K  }  +  \rbr{ \iota H^5   |S|K   C }^{\frac{1}{3}},  \sqrt{LH^3 \iota |S| K }}    } .
\end{align*}
\end{theorem}

Similar results for an adaptive opponent can be found in \pref{thm:Bound_final_adaptive}.
Once again, we emphasize that our algorithm is completely parameter-free and its regret bound on $\ENR_K$ smoothly interpolates between $\otil(\sqrt{K})$ and $\otil(K^{2/3})$ when the non-stationarity measures $C$ and $L$ increase.
This means that it recovers both the $\ExtR_K=\otil(\sqrt{K})$ external regret bound for a fixed opponent as achieved by standard single-agent RL algorithms and also the $\NR_K=\otil(K^{2/3})$ Nash-value regret bound of \citet{tian2021online} in the worst case.
Beyond the worst case, our bound can be strictly sharper when the opponent is moderately non-stationary. For example, if $C=o(K)$, then our regret is $o(K^{2/3})$, and if the opponent switches policies only logarithmically many times, then our regret scales as $\otil(\sqrt{K})$.

\section{Conclusion and Open Problems}

In this paper, we study online learning in uninformed Markov games, where the learner does not observe the opponent’s actions. We introduced ENR, a new regret notion that is strictly stronger than NR used in \cite{tian2021online} and that reduces to standard external regret when the opponent plays a fixed policy.
Under ENR, we provide a new analysis of epoch-based V-learning and show how its performance depends on a natural measure of opponent's non-stationarity. 
Building on this analysis, we design a parameter-free meta-algorithm that adaptively restarts epoch V-learning and achieves a regret bound that smoothly interpolates between the stationary and worst-case regimes, attaining $\otil(\min\{\sqrt{K}+(CK)^{1/3},\sqrt{LK}\})$ regret bound.
Our results also elicit compelling open questions:
\begin{itemize}[leftmargin=*]
 \item \textbf{Is the worst-case $\otil(K^{2/3})$ bound tight in uninformed MGs?}
There remains a gap between the current $\otil(K^{2/3})$ upper bound and the best known $\Omega(\sqrt{K})$ lower bound.
A trivial baseline is obtained by enumerating all $|A|^{|S|H}$ deterministic policies and running an adversarial bandit algorithm over this class, which yields $\otil(H\sqrt{|A|^{|S|H}K})$ regret.
Therefore, closing the gap requires either a lower bound of order $\Omega(\min\{H\sqrt{|A|^{|S|H}K},K^{2/3}\})$ or a genuinely new algorithmic or analytical idea that improves the $\otil(K^{2/3})$ rate.

\item \textbf{Can $\otil(\sqrt{LK})$ be improved under \ENR{}?}
If one only aims for a $\otil(\sqrt{LK})$ guarantee, the algorithm of \citet{wei2021non} can be applied to obtain such a bound.
However, their result is stated for dynamic regret, which is stronger than \ENR{}.
Since we work with the weaker notion \ENR{}, it is natural to ask whether one can obtain a better dependence on $L$ under \ENR{} while still degrading gracefully to the worst-case $\otil(K^{2/3})$ rate when $L$ is on the order of $K$.

    \item  \textbf{A ``best-of-all-worlds'' guarantee across regret notions.}
The setting we consider generalizes online learning in a sequence of changing MDPs, yet our guarantees are proved only for \ENR{}.
In contrast, \citet{wei2021non} and \citet{jin2023no} study changing MDPs under different regret notions.
They consider dynamic regret and transition adaptive external regret, respectively, and both notions are stronger than \ENR{}.
This raises a natural open question: can we design a single algorithm that, without knowing the regime in advance, attains our sublinear \ENR{} guarantee in the worst case, while also matching the sharper guarantees of \citet{wei2021non} and \citet{jin2023no} under their respective regret notions whenever the interaction reduces to their changing MDP models?

\end{itemize}

\paragraph{Acknowledgment}
HL is supported by NSF award IIS-1943607. LJR and JL are supported in part by NSF award AF-2312775 and CPS-1844729.

\bibliography{refs} %

\newpage
\clearpage
\appendix
\begingroup
\part{Appendix}
\let\clearpage\relax
\parttoc[n]
\endgroup

\newpage

\section{Regret Bound of Epoch V-learning for Oblivious Opponent}

\begin{remark} \label{rem:discussion_reg}
We remark that what \citet{tian2021online} obtain for the case with a fixed opponent is subtle.
Applying their result literally, one indeed only gets an $\otil(K^{2/3})$ bound on the Nash-value regret instead of the external regret.
However, if one applies their result to a different MG, then it in fact does imply an $\otil(K^{2/3})$ bound on the external regret.
To see this, note that since the opponent is fixed, the learner is equivalently facing a fixed MDP, which in turn can be treated as another MG with a dummy min-player with only one action in all states.
It is clear that Nash-value regret for this MG is the external regret of the original MG.

Nevertheless, we argue that this type of reasoning is cumbersome and, more importantly, does not generalize as long as the opponent deviates slightly from the ideal stationary case.
On the contrary, our proposed notion of empirical Nash-value regret is much more convenient and versatile.
\end{remark}

\subsection{Construction of Nice Event $\calG$}

\begin{lemma}[Restatement of \pref{lem:bound_num_epoch}]
For any $h \in [H]$ and $s \in S_h$, $|\calE(h,s)|\leq \left \lceil \frac{(1+\eta)\log(K)}{\eta} \right \rceil$.
\end{lemma}
\begin{proof}
Consider an arbitrary pair $(h,s)$.
We will show that the total number of epochs cannot exceed $\tau=\lceil \log(K)/\log(1+\eta) \rceil$.
Suppose that $(h,s)$ is in epoch $\tau$.
As $N_0(h,s)=1$ for all $(h,s)$, we have $\lceil(1+\eta) N_{\tau-1}(h,s) \rceil \geq (1+\eta) N_{\tau-1}(h,s) \geq \cdots \geq (1+\eta)^{\tau}N_0(h,s)=(1+\eta)^{\tau}$. Thus,
\[
\lceil(1+\eta) N_{\tau-1}(h,s) \rceil \geq  \rbr{1+\eta}^{\tau} \geq  \rbr{1+\eta}^{\log(K)/\log(1+\eta)} = K,
\]
which implies that to end epoch $\tau$, $(h,s)$ should be visited over $K$ times. 
Using the fact that $\log(1+x) \geq \frac{x}{1+x}$ for any $x>-1$ to lower bound $\log(1+\eta) \geq \eta/(1+\eta)$, we complete the proof.
\end{proof}

Given \pref{lem:bound_num_epoch}, we define the upper bound of epoch number as
\begin{equation} \label{eq:def_M_upper_bound_epoch}
    M:= \left \lceil \frac{(1+\eta)\log(K)}{\eta} \right \rceil.
\end{equation}

For any $h \in [H+1]$, we define function $F_{h}:S_{h} \to [0,H]$, and define
\begin{align}
    \epsilon_h^k \rbr{s;F_{h+1} }:= \mbD_{\mu^k_h,\nu^k_h} \sbr{ r_h+P_hF_{h+1} }(s) -\rbr{ r_h^k+F_{h+1}(s_{h+1}^k)}.
\end{align}

We choose $\iota$ as:
\begin{align} \label{eq:exact_choice_iota}
    \iota = \max\{4c^2,8\}H^2|A|\log(8HKM|A||S|/\delta).
\end{align}

Recall that $\calF_{h}^k$ is the history before step $h$ in episode $k$.
We define events
\begin{align}
   \calG^1 &:= \cbr{ \forall h : \abr{ \sum_{k=1}^K \E \sbr{\delta^k_{h} \mid \calF_{h}^k } - \sum_{k=1}^K \delta^k_{h}} \leq 2H \sqrt{2K\log(8H/\delta) } } , \label{eq:def_event1} \\
   \calG^2 &:= \cbr{ \forall h,s,\tau : \abr{ \sum_{t \in \calK_{\tau}(h,s)} \epsilon_h^t \rbr{s;V_{h+1}^t}   } \leq \frac{1}{2} N_{\tau}(h,s) \beta_{N_{\tau}(h,s)} }, \label{eq:def_event2} \\
    \calG^3 &:= \cbr{ \forall h,s,\tau : \abr{ \sum_{t \in \calK_{\tau}(h,s)} \epsilon_h^t \rbr{s;V_{h+1}^*}  } \leq \frac{1}{2} N_{\tau}(h,s) \beta_{N_{\tau}(h,s)} }, \label{eq:def_event3} \\
      \calG^4 &:= \cbr{ \forall h,s,\tau : \Reg^{\tau}_{h,s}(N_{\tau}(h,s))  \leq \frac{1}{2H} N_{\tau}(h,s) \beta_{N_{\tau}(h,s)}  } , \label{eq:def_event4}
\end{align}
where $c >0$ is some absolute constant and $\Reg^{\tau}_{h,s}(N_{\tau}(h,s)) $ is defined in \pref{eq:ext_reg_adv_bandit}.

\begin{definition} \label{def:high_prob_calE}
Let $\calG$ be the event such that $\calG=\calG^1 \cap \calG^2 \cap \calG^3 \cap \calG^4$.
\end{definition}
\begin{lemma}\label{lem:nice_event_epochvol_oblivious}
We have $\P(\calG) \geq 1-\delta$.
\end{lemma}
\begin{proof}
For $\calG^1$, one can see that $|\E \sbr{\delta^k_{h} \mid \calF_{h}^k }-\delta^k_{h}| \leq 2H$ for all $k,h$. We apply Azuma–Hoeffding for martingale difference sequence $\{\E \sbr{\delta^k_{h} \mid \calF_{h}^k }-\delta^k_{h}\}_{k=1}^K$ with respect to any fixed step $h \in [H]$, and then use a union bound over all $h \in [H]$ to obtain $\P(\calG^1) \geq 1-\delta/4$.

For $\calG^2$, we rewrite $ \sum_{t \in \calK_{\tau}(h,s)} \epsilon_h^t \rbr{s;V_{h+1}^t}=\sum_{i=1}^{ N_{\tau}(h,s)} \epsilon_h^{t_i} \rbr{s;V_{h+1}^{t_i}}$ where $t_i$ is the episode, corresponding to the $i$-th visit of $(h,s)$ pair in epoch $\tau$. 
We have that $\E [\epsilon_h^{t_i} \rbr{s;V_{h+1}^{t_i}} \mid \calF_{h}^{t_i}]=0$ and $|\epsilon_h^{t_i} \rbr{s;V_{h+1}^{t_i}}| \leq H$. Consider fixed $(h,s)$ and epoch $\tau$, and we further fix a $N_{\tau}(h,s) \in [K]$. Applying Azuma–Hoeffding gives that with probability at least $1-\delta/(4HKM|S|)$,
\[
\abr{\sum_{i=1}^{ N_{\tau}(h,s) } \epsilon_h^{t_i} \rbr{s;V_{h+1}^{t_i}}} \leq H\sqrt{2N_{\tau}(h,s) \log(8KH|S|M/\delta)} \leq \frac{1}{2} N_{\tau}(h,s) \beta_{N_{\tau}(h,s)} .
\]

Using a union bound over all $(h,s)$, $\tau$, $N_{\tau}(h,s)$, and following \pref{lem:bound_num_epoch} that the number of epochs of any $(h,s)$ pair is at most $M$, we have $\P(\calG^2) \geq 1-\delta/4$.

The argument for proving $\P(\calG^3) \geq 1-\delta/4$ is analogous to that of $\calG^2$. One caveat here is that we can apply Azuma-Hoeffding inequality because the opponent is assumed to be oblivious.

Then, we show that $\P(\calG^4) \geq 1-\delta/4$. For any $(h,s)$ pair, \pref{alg:epoch_V_ol} runs a new instance of FTRL with $1/2$-Tsallis, IX estimator, and doubling trick in each epoch. By \citep[Theorem 2]{lecture13}, running this algorithm ensures that with probability at least $1-\delta/(4H|S|M)$, for any $N_{\tau}(h,s) \in \naturalnum$
\[
\Reg^{\tau}_{h,s}(N_{\tau}(h,s))\leq c\sqrt{|A_h|N_{\tau}(h,s) \log(HM|A_h||S|/\delta) } \leq \frac{1}{2H} N_{\tau}(h,s) \beta_{N_{\tau}(h,s)} ,
\]
where $c>0$ is some absolute constant.
Using a union bound over all $h,s,\tau$ gives the claimed result.

Finally, using a union bound over $\calG^1,\calG^2,\calG^3,\calG^4$ completes the proof.
\end{proof}

\subsection{Proof of \pref{thm:reg_bound_evol}}

The proof of \pref{thm:reg_bound_evol} conditions on event $\calG$ defined in \pref{def:high_prob_calE}, which occurs with probability at least $1-\delta$.
By the optimism in \pref{lem:optimism} and the definition $\delta^k_h =  (V^{k}_h -V^{\mu^k,\nu^k}_h)(s_h^k)$, we write
\begin{align*}
 \ENR_K \leq   \sum_{k=1}^K \rbr{ V^{k}_1(s_1^k) -V^{\mu^k,\nu^k}_1(s_1^k)  } = \sum_{k=1}^K  \delta^k_1.
\end{align*}

To bound $\delta^k_1$, we analyze every $\delta^k_h$.
Consider any fixed episode $k$ and step $h$. By subtracting and adding $\mbD_{\mu^k_h,\nu^k_h} \sbr{r_h+P_h V^{k}_{h+1} }(s_h^k) $, we write
\begin{align*}
    \delta^k_h   &=V^{k}_h(s_h^k)  - \mbD_{\mu^k_h,\nu^k_h} \sbr{r_h + P_h V^{\mu^k,\nu^k}_{h+1} }(s_h^k) \\
    & =   \underbrace{\rbr{V^{k}_h(s_h^k) - \mbD_{\mu^k_h,\nu^k_h} \sbr{r_h+P_h V^{k}_{h+1} }(s_h^k)  }   }_{(1)_{h}^k}+  \underbrace{ \mbD_{\mu^k_h,\nu^k_h} \sbr{P_h (V^{k}_{h+1}-V^{\mu^k,\nu^k}_{h+1}) }(s_h^k) }_{(2)_{h}^k}.
\end{align*}

With this decomposition, for any step $h \in [H]$,
\begin{align} \label{eq:decompose_delta_oneplustwo}
    \sum_{k=1}^K \delta_h^k = \sum_{k=1}^K (1)_h^k+\sum_{k=1}^K (2)_h^k .
\end{align}

\textbf{Bounding $ \sum_{k=1}^K (2)_{h}^k$.}
Recall that $\calF_{h}^k$ is the history before step $h$ in episode $k$.
We have
\begin{align*}
\E \sbr{\delta^k_{h+1}  \mid \calF_{h+1}^k }= \E \sbr{ (V^{k}_{h+1} -V^{\mu^k,\nu^k}_{h+1})(s_{h+1}^k) \mid \calF_{h+1}^k } =\mbD_{\mu^k_h,\nu^k_h} \sbr{P_h (V^{k}_{h+1}-V^{\mu^k,\nu^k}_{h+1})  }(s_h^k).
\end{align*}

Then, $(2)_{h}^k$ can be bounded by
\begin{align*}
    (2)_{h}^k &= \delta^k_{h+1}+  \mbD_{\mu^k_h,\nu^k_h} \sbr{P_h (V^{k}_{h+1}-V^{\mu^k,\nu^k}_{h+1})  }(s_h^k) - \delta^k_{h+1} 
     = \delta^k_{h+1}+ \E \sbr{\delta^k_{h+1} \mid \calF_{h+1}^k } - \delta^k_{h+1} .
\end{align*}

We use \pref{eq:def_event1} of $\calG$ to bound $\sum_{k=1}^K \rbr{\E \sbr{\delta^k_{h+1} \mid \calF_{h+1}^k } - \delta^k_{h+1}}$ to obtain
\begin{align} \label{eq:bound_term_twohk}
    \sum_{k=1}^K   (2)_{h}^k \leq 2H \sqrt{2K\log(8H/\delta) } + \sum_{k=1}^K  \delta^k_{h+1}.
\end{align}

\textbf{Bounding $ \sum_{k=1}^K (1)_{h}^k$.} 
By the definition of $L_{\tau}(h,s)$ in \pref{eq:non_increasing_update}, we have $V^k_h(s)=L_{\tau-1}(h,s)$ for all $k \in K_{\tau}(h,s)$.
For any step $h \in [H]$, one can write
\begin{align*}
\sum_{k=1}^K (1)_h^k & =  \sum_{k=1}^K \rbr{V^{k}_h(s_h^k) - \mbD_{\mu^k_h,\nu^k_h} \sbr{r_h+P_h V^{k}_{h+1} }(s_h^k) } \\
& =  \sum_{s \in S_h} \sum_{\tau \in \calE(h,s)} \sum_{t \in K_{\tau}(h,s)} \rbr{V^{t}_h(s) - \mbD_{\mu^t_h,\nu^t_h} \sbr{r_h+P_h V^{t}_{h+1} }(s) } \\
& \leq  \sum_{s \in S_h} \sum_{\tau \in \calE(h,s)} \sum_{t \in K_{\tau}(h,s)} \rbr{V^{t}_h(s) - \rbr{r_h^t +V^t_{h+1}(s_{h+1}^t)  }  }  + \frac{1}{2} \sum_{s \in S_h} \sum_{\tau \in \calE(h,s)} N_{\tau}(h,s) \beta_{N_{\tau}(h,s)} \\
& =  \sum_{s \in S_h} \sum_{\tau \in \calE(h,s)} \sum_{t \in K_{\tau}(h,s)} \rbr{L_{\tau-1}(h,s) - \rbr{r_h^t +V^t_{h+1}(s_{h+1}^t)  }  }  + \frac{1}{2} \sum_{s \in S_h} \sum_{\tau \in \calE(h,s)} N_{\tau}(h,s) \beta_{N_{\tau}(h,s)} \\
& \leq  \sum_{s \in S_h} \sum_{\tau \in \calE(h,s)}N_{\tau}(h,s)  \rbr{ L_{\tau-1}(h,s) -L_{\tau}(h,s)   } + \frac{3}{2}  \sum_{s \in S_h} \sum_{\tau \in \calE(h,s)} N_{\tau}(h,s) \beta_{N_{\tau}(h,s)} , \numberthis{} \label{eq:bound_one_two_terms}
\end{align*}
where the first inequality uses \pref{eq:def_event2} of $\calG$, the last equality holds since for any $t \in K_{\tau}(h,s)$, we have $V^{t}_h(s) = L_{\tau-1}(h,s)$, and the last inequality holds due to:
\begin{align*}
 &\sum_{t \in K_{\tau}(h,s)}    \rbr{r_h^t +V^t_{h+1}(s_{h+1}^t)  } \\
 &=N_{\tau}(h,s) \rbr{\frac{1}{N_{\tau}(h,s)}\sum_{t \in K_{\tau}(h,s)}    \rbr{r_h^t + V^t_{h+1}(s_{h+1}^t)  } +\beta_{N_{\tau}(h,s)} } - N_{\tau}(h,s) \beta_{N_{\tau}(h,s)} \\
  &\geq N_{\tau}(h,s) L_{\tau}(h,s) - N_{\tau}(h,s) \beta_{N_{\tau}(h,s)}.
\end{align*}

We use \pref{lem:Bounds_N_beta} to bound the second term of \pref{eq:bound_one_two_terms} as:
\begin{equation} \label{eq:second_term_bound}
    \sum_{s \in S_h} \sum_{\tau \in \calE(h,s)} N_{\tau}(h,s) \beta_{N_{\tau}(h,s)} \leq   \order \rbr{  \sqrt{ \frac{ \iota |S|K \log(K) }{\eta} }    }.
\end{equation}

Then, we turn to bound the first term of \pref{eq:bound_one_two_terms} as:
\begin{align*}
&  \sum_{s \in S_h} \sum_{\tau \in \calE(h,s)}N_{\tau}(h,s)  \rbr{ L_{\tau-1}(h,s) -L_{\tau}(h,s)   }  \\
&\overset{(a)}{=}  \sum_{s \in S_h}  \rbr{N_1(h,s)L_0(h,s) + \sum_{\tau=1}^{|\calE(h,s)|-1} L_{\tau}(h,s)  \rbr{ N_{\tau+1}(h,s)-N_{\tau}(h,s) } -N_{|\calE(h,s)|}(h,s)L_{|\calE(h,s)|}(h,s)  } \\
&=  \sum_{s \in S_h}  \rbr{N_1(h,s) \rbr{L_0(h,s) -V^*_h(s)} +  \sum_{\tau=1}^{|\calE(h,s)|-1}\rbr{ N_{\tau+1}(h,s)-N_{\tau}(h,s) } \rbr{  L_{\tau}(h,s) - V^*_h(s)} }\\
&\quad - \sum_{s \in S_h} N_{|\calE(h,s)|}(h,s) \rbr{ L_{|\calE(h,s)|}(h,s) -V^*_h(s)}  \\
&\overset{(b)}{\leq}   \sum_{s \in S_h}  \rbr{N_1(h,s) \rbr{L_0(h,s) - V^*_h(s)} +  \eta \sum_{\tau=1}^{|\calE(h,s)|-1} N_{\tau}(h,s) \rbr{  L_{\tau}(h,s) - V^*_h(s)} + |\calE(h,s)|H }\\
&\quad - \sum_{s \in S_h} N_{|\calE(h,s)|}(h,s) \rbr{ L_{|\calE(h,s)|}(h,s) -V^*_h(s)}  \\
&=  \sum_{s \in S_h}  \rbr{N_1(h,s) \rbr{L_0(h,s) -V^*_h(s)} + \eta \sum_{\tau \in \calE(h,s)}N_{\tau}(h,s) \rbr{  L_{\tau}(h,s) - V^*_h(s)} + |\calE(h,s)|H  }\\
&\quad - (1+\eta)\sum_{s \in S_h} N_{|\calE(h,s)|}(h,s) \rbr{ L_{|\calE(h,s)|}(h,s) -V^*_h(s)}  \\
&\overset{(c)}{\leq} H |S| N_1(h,s)+ \eta \sum_{s \in S_h}  \sum_{\tau \in \calE(h,s)}N_{\tau}(h,s) \rbr{  L_{\tau}(h,s) -V^*_h(s)} + \order \rbr{\frac{H|S| \log(K)}{\eta}}  \\
&\leq  \order \rbr{\eta  H^2 C
+ H \sqrt{\eta \iota |S|  K\log(K)  } +\frac{H |S| \log(K)}{\eta} }  ,\numberthis{}  \label{eq:first_term_bound}
\end{align*}
where step $(a)$ follows from the fact that for any $\{a_i\}_{i=1}^n,\{b_i\}_{i=1}^n$, we have $\sum_{i=1}^n a_i(b_{i-1}-b_i)=a_1b_0 + \sum_{i=1}^{n-1} (a_{i+1}-a_i)b_i -a_nb_n$, step $(b)$ uses facts that $N_{\tau+1}(h,s) -N_{\tau}(h,s) \leq \eta N_{\tau}(h,s) +1$ and $L_{\tau}(h,s)\geq V^*_h(s)$ by \pref{lem:optimism}, and step $(c)$ uses again $L_{\tau}(h,s) \geq V^{*}_h(s)$ and \pref{lem:bound_num_epoch} to bound $|\calE(h,s)| \leq \order (\log(K)/\eta )$ for all $(h,s)$, and the last inequality invokes \pref{lem:gamma_bound} and uses $\eta \leq 1/H$ to bound $H^2|S|\log(K) \leq H|S|\log(K)/\eta$.

Plugging \pref{eq:first_term_bound} and \pref{eq:second_term_bound} into \pref{eq:bound_one_two_terms}, we have that for any step $h \in [H]$
\begin{align} \label{eq:bound_term_onehk}
   \sum_{k=1}^K (1)_h^k \leq \order \rbr{ \eta H^2 C 
+ H \sqrt{\eta \iota |S|  K\log(K) } +\sqrt{ \frac{ \iota |S|K\log(K)}{\eta}  }+\frac{ H |S| \log(K)}{\eta}  }.
\end{align}

\textbf{Putting together.} Plugging \pref{eq:bound_term_onehk} and \pref{eq:bound_term_twohk} into \pref{eq:decompose_delta_oneplustwo}, we have that for any step $h \in [H]$
\begin{align*}
    &\sum_{k=1}^K \delta^k_h \leq \sum_{k=1}^K \delta^k_{h+1}  + \order \rbr{  \eta H^2 C 
+ H  \sqrt{\eta \iota |S|  K\log(K) } +\sqrt{ \frac{ \iota |S|K\log(K)}{\eta}  }+\frac{ H |S| \log(K)}{\eta}  }.
\end{align*}

Summing over all $h \in [H]$, using the fact $\delta_{H+1}^k=0$ for any $k$, and rearranging, we have 
\begin{align*}
\ENR_K \leq \sum_{k=1}^K \delta_1^k &\leq \order \rbr{ \eta H^3C
 + H^2 \sqrt{\eta \iota  |S| K\log(K) } +H\sqrt{ \frac{\iota |S|K\log(K)}{\eta}  } +\frac{H^2|S| \log(K)}{\eta} }\\
& \leq \order \rbr{ \eta H^3C 
+ H\sqrt{ \frac{\iota |S|K\log(K)}{\eta}  } +\frac{H^2|S| \log(K)}{\eta} },
\end{align*}
where the second inequality uses the fact $\eta \leq 1/H$ to bound $H^2  \sqrt{\eta \iota |S|  K\log(K) } \leq H\sqrt{ \frac{\iota |S|K\log(K)}{\eta}  }$.

Furthermore, if one assumes $K \geq H|S|$ and constrains $\eta \in \sbr{\frac{|S|}{K} , \frac{1}{H }}$, then, $\frac{H^2|S| \log(K)}{\eta} \leq H\sqrt{ \frac{\iota |S|K\log(K)}{\eta}  }$, which gives
\begin{align*}
\ENR_K \leq \sum_{k=1}^K \rbr{V^k_1(s_1^k)-V^{\mu^k,\nu^k}_1(s_1^k)} \leq \order \rbr{ \eta H^3C
+ H\sqrt{ \frac{\iota |S|K\log(K)}{\eta}  }   }.
\end{align*}

\subsection{Supporting Lemmas}
\label{app:supporting_lemma_epoch_vol}

\begin{lemma}[Optimism]\label{lem:optimism}
Suppose that $\calG$ holds where $\calG$ is defined in \pref{def:high_prob_calE}.
For all $h \in [H+1]$, we have $V_h^{*}(s) \leq V_h^k(s)$ for all  $k\in [K],s \in S_h$. 
\end{lemma}
\begin{proof}
Here, we use backward induction on $h$ to show that for all $h$, $V^*_h  \leq V^{k}_h$ holds in entry-wise for all episode $k$. For the base case $h=H+1$, $V^*_h  \leq V^{k}_h$ holds for all $k \in [K]$ by the definitions that $V^*_{H+1}(s)=V^{k}_{H+1}(s)=0$ for all state $s$ and episode $k$. Suppose that for step $h+1$, we have $V^*_{h+1}  \leq V^{k}_{h+1}$ for all $k \in [K]$, and then we show that at step $h$, $V^*_h  \leq V^{k}_h$ holds for all $k \in [K]$. 

Let $\tau(k,h,s)$ be the epoch that $(h,s)$ pair lies at episode $k$. 
Consider any fixed episode $k \in [K]$, and any fixed state $s \in S_h$. Notice that if $\tau(k,h,s)=1$, then $V^k_h(s)=H-h+1$, directly implying $V^*_h(s) \leq V^k_h(s)$. In the following, we consider $\tau(k,h,s)>1$ and write
\begin{align*}
V^*_h(s) &= \max_{\mu  \in \Delta(A_h)} \min_{\nu \in \{\nu^1_h(s),\ldots,\nu^K_h(s)\}  } \mbD_{\mu,\nu} \sbr{ r_h +P_hV_{h+1}^{*} }(s) \\
&=\frac{1}{N_{\tau(k,h,s)-1}(h,s)} \max_{\mu \in \Delta(A_h)} \sum_{t \in \calK_{\tau(k,h,s)-1} (h,s) }  \min_{\nu \in \{\nu^1_h(s),\ldots,\nu^K_h(s)\}} \mbD_{\mu,\nu} \sbr{ r_h +P_hV_{h+1}^{*} }(s) \\
&\leq \frac{1}{N_{\tau(k,h,s)-1}(h,s)} \max_{\mu \in \Delta(A_h)}\sum_{t \in \calK_{\tau(k,h,s)-1} (h,s) }    \mbD_{\mu,\nu^{t}_h} \sbr{ r_h +P_hV_{h+1}^{*} }(s) \\
&\leq \frac{1}{N_{\tau(k,h,s)-1}(h,s)}  \max_{\mu \in \Delta(A_h)} \sum_{t \in \calK_{\tau(k,h,s)-1} (h,s) }  \mbD_{\mu,\nu^{t}_h} \sbr{ r_h +P_hV_{h+1}^{t} }(s) \\
&\leq \frac{1}{N_{\tau(k,h,s)-1}(h,s)}   \sum_{t \in \calK_{\tau(k,h,s)-1} (h,s) }  \mbD_{\mu^{t}_h,\nu^{t}_h} \sbr{ r_h +P_hV_{h+1}^{t} }(s) +\frac{1}{2} \beta_{N_{\tau(k,h,s)-1}(h,s)} \\
&\leq \frac{1}{N_{\tau(k,h,s)-1}(h,s)} \sum_{t \in \calK_{\tau(k,h,s)-1} (h,s) } \rbr{r_h^t +V_{h+1}^{t}(s_{h+1}^t) }+\beta_{N_{\tau(k,h,s)-1}(h,s)} , \numberthis{} \label{eq:needs_justify}
\end{align*}
where the second inequality uses the induction hypothesis, and the third inequality follows from \pref{eq:def_event4} of event $\calG$, and the last inequality uses \pref{eq:def_event2} of event $\calG$.

Since $V^*_h(s) \leq H-h+1$, combining this and \pref{eq:needs_justify}, we have $V^*_h(s) \leq V^k_h(s)$.
As this argument holds for all episodes $k$ and states $s \in S_h$, the induction is done, and the lemma thus follows.
\end{proof}

We define the variance of opponent's policy in epoch $\tau$ of $(h,s)$ pair as:
\begin{align}
    C_{\tau}(h,s) =  \sum_{t \in K_{\tau}(h,s)} \TV \rbr{\nu^t_h(s),\nu^*_h(s)  }  .
\end{align}

\begin{lemma} \label{lem:Bounds_N_beta}
Suppose that $\eta \in (0,1/H]$.
For any step $h \in [H]$, the following holds.
\begin{align*}
 &  \sum_{s \in S_h} \sum_{\tau \in \calE(h,s)} N_{\tau}(h,s) \beta_{N_{\tau}(h,s)} \leq \order \rbr{  \sqrt{ \frac{  \iota |S|K \log(K) }{\eta} }    }.
\end{align*}
\end{lemma}

\begin{proof}
We show that
\begin{align*}
&\sum_{s \in S_h} \sum_{\tau \in \calE(h,s)} N_{\tau}(h,s) \beta_{N_{\tau}(h,s)} \\
&=   \sum_{s \in S_h} \sum_{\tau \in \calE(h,s)}  \sqrt{ \iota N_{\tau}(h,s) }\\
 &\leq   \sum_{s \in S_h}  \sqrt{ \iota |\calE(h,s)| \rbr{\sum_{\tau \in \calE(h,s)}  N_{\tau}(h,s) }}\\
 &\leq   \sum_{s \in S_h}  \sqrt{ \iota M \rbr{\sum_{\tau \in \calE(h,s)}  N_{\tau}(h,s) }}\\
 &\leq   \sqrt{ \iota M  |S_h|K}\\
&\leq \order \rbr{  \sqrt{ \frac{\iota |S|K \log(K) }{\eta} }    },
\end{align*}
where the first inequality applies the Cauchy–Schwarz inequality, the second inequality uses \pref{lem:bound_num_epoch} to bound $|\calE(h,s)|\leq M$ where $M$ is given in \pref{eq:def_M_upper_bound_epoch}, the third inequality uses again the Cauchy–Schwarz inequality and bounds $\sum_{s \in S_h} \sum_{\tau \in \calE(h,s)}  N_{\tau}(h,s) \leq K$, and the last inequality follows from the definition of $M$ and simply bounds $|S_h| \leq |S|$.

The proof is thus complete.
\end{proof}

\begin{lemma} \label{lem:L_minus_Vstar}
Suppose that $\calG$ holds where $\calG$ is defined in \pref{def:high_prob_calE}.
For any pair $(h,s)$ and any epoch $\tau$ of this pair, we have
\begin{align*} 
&  \sum_{t \in \calK_{\tau}(h,s)} \rbr{r_h^t + V^{t}_{h+1}(s^t_{h+1}) -V^{*}_h(s) } \\
&\leq    \order \rbr{H C_{\tau}(h,s) +  N_{\tau}(h,s)\beta_{N_{\tau}(h,s)} }   + \sum_{t \in \calK_{\tau}(h,s)} \rbr{ V^{t}_{h+1}(s^t_{h+1}) - V^{*}_{h+1}(s^t_{h+1})  }.
\end{align*}
\end{lemma}

\begin{proof}
We show that for $(h,s)$ pair and any epoch $\tau$
\begin{align*}
& 
 \sum_{t \in \calK_{\tau}(h,s)} \rbr{r_h^t +V^{t}_{h+1}(s^t_{h+1}) - V^{*}_h(s) }   \\
&= 
\sum_{t \in \calK_{\tau}(h,s)} \rbr{r_h^t +V^{*}_{h+1}(s^t_{h+1})-V^{*}_h(s) } + \sum_{t \in \calK_{\tau}(h,s)} \rbr{ V^{t}_{h+1}(s^t_{h+1}) - V^{*}_{h+1}(s^t_{h+1})  }   .
\numberthis{} \label{eq:above_bound_1}
\end{align*}

Then, we have
\begin{align*}
&  \sum_{t \in \calK_{\tau}(h,s)} \rbr{r_h^t +V^{*}_{h+1}(s^t_{h+1})-V^{*}_h(s) } \\
&\leq \sum_{t \in \calK_{\tau}(h,s)} \rbr{\mbD_{\mu^t_h,\nu^t_h} \sbr{ r_h + P_h V^{*}_{h+1}}(s) -V^{*}_h(s)}  + \order \rbr{   N_{\tau}(h,s)\beta_{N_{\tau}(h,s)} } \\
&= \sum_{t \in \calK_{\tau}(h,s)} \rbr{ \mbD_{\mu^t_h,\nu^t_h} \sbr{ r_h + P_h V^{*}_{h+1}}(s) - \mbD_{\mu^*_h,\nu^*_h} \sbr{ r_h + P_h V^{*}_{h+1} }(s)}  + \order \rbr{ N_{\tau}(h,s)\beta_{N_{\tau}(h,s)} } \\
&\leq \sum_{t \in \calK_{\tau}(h,s)} \rbr{ \mbD_{\mu^t_h,\nu^t_h} \sbr{ r_h + P_h V^{*}_{h+1}}(s) - \mbD_{\mu^t_h, \nu^*_h } \sbr{ r_h + P_h V^{*}_{h+1} }(s)} + \order \rbr{  N_{\tau}(h,s)\beta_{N_{\tau}(h,s)} } ,
\end{align*}
where the first inequality uses \pref{eq:def_event3} of $\calG$ and the last inequality follows from the following fact that for any $h,s,t$
\begin{align*}
&\mbD_{\mu^*_h,\nu^*_h} \sbr{ r_h + P_h V^{*}_{h+1} }(s)\\
&=     \max_{\mu \in \Delta(A_h)}  \min_{ \nu \in  \cbr{\nu^1_h(s),\ldots,\nu^K_h(s) }   } \mbD_{\mu,\nu}\sbr{ r_h + P_h V^{*}_{h+1} }(s) \\
&=    \max_{\mu \in \Delta(A_h)} \min_{ \nu \in \Conv \rbr{\nu^1_h(s),\ldots,\nu^K_h(s) }   }   \mbD_{\mu,\nu}  \sbr{ r_h + P_h V^{*}_{h+1} }(s) \\
&=   \min_{ \nu \in \Conv \rbr{\nu^1_h(s),\ldots,\nu^K_h(s) }   } \max_{\mu \in \Delta(A_h)}    \mbD_{\mu,\nu}  \sbr{ r_h + P_h V^{*}_{h+1} }(s) \\
&=  \max_{\mu \in \Delta(A_h)}    \mbD_{\mu,\nu^*_h}  \sbr{ r_h + P_h V^{*}_{h+1} }(s) \\
&\geq     \mbD_{\mu^t_h,\nu^*_h}  \sbr{ r_h + P_h V^{*}_{h+1} }(s),
\end{align*}
where the second equality holds since for any $\mu$, $\mbD_{\mu,\nu}\sbr{ r_h + P_h V^{*}_{h+1} }(s)$ is linear in $\nu$, the third equality follows from the minimax theorem, and the fourth equality uses the definition of $\nu^*_h$.

Let us define
\begin{align*}
      g^t_{h,s}(b) = \E_{a \sim \mu^t_h} \sbr{  r_h(s,a,b) +\E_{s' \sim P_h(\cdot |s,a,b)} \sbr{ V^*_{h+1}(s') }  } \in [0,H].
\end{align*}

Then, we have
\begin{align*}
&\sum_{t \in \calK_{\tau}(h,s)}\rbr{ \mbD_{\mu^t_h,\nu^t_h} \sbr{ r_h + P_h V^{*}_{h+1}}(s) - \mbD_{\mu^t_h,\nu^*_h} \sbr{ r_h + P_h V^{*}_{h+1} }(s)}   \\
&= \sum_{t \in \calK_{\tau}(h,s)}\rbr{  \E_{b \sim \nu^t_h}[  g^t_{h,s}(b) ] - \E_{b \sim \nu^*_h}[  g^t_{h,s}(b) ]   } \\
&\leq 2\sum_{t \in \calK_{\tau}(h,s)}  \norm{g^t_{h,s}(\cdot)}_{\infty} \TV(\nu^t_h(s),\nu^*_h(s)) \\
&\leq 2  H\sum_{t \in \calK_{\tau}(h,s)}  \TV(\nu^t_h(s),\nu^*_h(s))\\
&= 2HC_{\tau}(h,s). \numberthis{} \label{eq:bound_epoch_corruption}
\end{align*}

The claimed result thus follows.
\end{proof}

We prove the following lemma, which states the claim of \pref{lem:gamma_bound} in a slightly different way, i.e., we prove the bound under a nice event $\calG$, which satisfies $\P(\calG) \geq 1-\delta$.

\begin{lemma}\label{lem:gamma_bound_oblivious_app_version}
Suppose that $\calG$ holds where $\calG$ is defined in \pref{def:high_prob_calE} and $\eta \in (0,1/H]$.
For any step $h \in [H]$, we have
\begin{align*}
&\sum_{s \in S_h} \sum_{\tau \in \calE(h,s)}N_{\tau}(h,s) \rbr{  L_{\tau}(h,s) -V^*_h(s)} \leq \order \rbr{ H^2 C
 +  H\sqrt{ \frac{\iota  |S| K\log(K)}{\eta}  } + \frac{H^2|S|\log(K)}{\eta}   } .
\end{align*}
\end{lemma}

\begin{proof}
For shorthand, let
\begin{align*}
      \Gamma(h) :=\sum_{s \in S_h} \sum_{\tau \in \calE(h,s)}N_{\tau}(h,s) \rbr{  L_{\tau}(h,s) - V^*_h(s)} .
\end{align*}

We then write
\begin{align*}
\Gamma(h) & =\sum_{s \in S_h} \sum_{\tau \in \calE(h,s)}  N_{\tau}(h,s)   \rbr{L_{\tau}(h,s)  - V^{*}_h(s) }  \\
& \leq \sum_{s \in S_h} \sum_{\tau \in \calE(h,s)}     \sum_{t \in K_{\tau}(h,s)} \rbr{r_h^t +V^t_{h+1}(s_{h+1}^t) - V^{*}_h(s) } +\sum_{s \in S_h} \sum_{\tau \in \calE(h,s)}  N_{\tau}(h,s)  \beta_{N_{\tau}(h,s)}     \\
&\leq   \sum_{s \in S_h}\sum_{\tau \in \calE(h,s)}  \order  \rbr{HC_{\tau}(h,s)+ N_{\tau}(h,s)\beta_{N_{\tau}(h,s)}  } \\
&\quad +  \sum_{s \in S_h}\sum_{\tau \in \calE(h,s)}  \sum_{t \in \calK_{\tau}(h,s)} \rbr{ V^{t}_{h+1}(s^t_{h+1}) - V^{*}_{h+1}(s^t_{h+1})  }  \\
&\leq   \order \rbr{ HC +  \sqrt{ \frac{\iota |S| K\log(K)}{\eta}  }   }   + \sum_{s \in S_h}\sum_{\tau \in \calE(h,s)} \sum_{t \in \calK_{\tau}(h,s)} \rbr{V^{t}_{h+1}(s^t_{h+1}) - V^{*}_{h+1}(s^t_{h+1})  } ,\numberthis{} \label{eq:Gamma_bound_step2}
\end{align*}
where the first inequality uses the definition of $L_{\tau}(h,s)$ in \pref{eq:non_increasing_update}, the second inequality uses \pref{lem:L_minus_Vstar}, and the last inequality uses \pref{lem:Bounds_N_beta}.

Now, we show that
\begin{align*}
& \sum_{s \in S_h}\sum_{\tau \in \calE(h,s)}  \sum_{t \in \calK_{\tau}(h,s)} \rbr{ V^{t}_{h+1}(s^t_{h+1}) - V^{*}_{h+1}(s^t_{h+1})  } \\
&= \sum_{s \in S_{h+1}}\sum_{\tau \in \calE(h+1,s)}  \sum_{t \in \calK_{\tau}(h+1,s)} \rbr{ V^{t}_{h+1}(s) - V^{*}_{h+1}(s)  } \\
&=  \sum_{s \in S_{h+1}}\sum_{\tau \in \calE(h+1,s)} N_{\tau}(h+1,s) \rbr{L_{\tau-1}(h+1,s) - V^{*}_{h+1}(s)  }
 \\
&\leq  \sum_{s \in S_{h+1}}\sum_{\tau \in \calE(h+1,s)} (1+\eta)N_{\tau-1}(h+1,s) \rbr{L_{\tau-1}(h+1,s) - V^{*}_{h+1}(s)  } +\order \rbr{ \frac{H|S|\log(K)}{\eta}}\\
&\leq  (1+\eta)\sum_{s \in S_{h+1}}\sum_{\tau \in \calE(h+1,s)} N_{\tau}(h+1,s) \rbr{L_{\tau}(h+1,s) - V^{*}_{h+1}(s)  } +\order \rbr{ \frac{H|S|\log(K)}{\eta}}
 \\
&= (1+\eta)\Gamma(h+1) +\order \rbr{ \frac{H|S|\log(K)}{\eta}},
\end{align*}
where the first inequality follows from the optimism (see \pref{lem:optimism}) that $L_{\tau}(h,s) \geq V^*_h(s)$ for all $\tau,h,s$, and the fact that $N_{\tau}(h,s) \leq (1+\eta)N_{\tau-1}(h,s)+1$, the second inequality uses $\forall (h,s)$, $N_0(h,s)=1$ and $L_0(h,s)=H-h+1$ to bound $(1+\eta)\sum_{s \in S_{h+1}}N_0(h+1,s)(L_0(h+1,s)-V^*_{h+1}(s)) \leq 2|S|H \leq \order(H|S|\log(K)/\eta)$.

By plugging the above into \pref{eq:Gamma_bound_step2} and using the facts that $\Gamma(H+1)=0$, we have
\begin{align*}
   \Gamma(h) &\leq (1+\eta)\Gamma(h+1) 
+  \order \rbr{HC  +  \sqrt{ \frac{\iota |S| K\log(K)}{\eta}  } + \frac{H|S|\log(K)}{\eta}}\\
 &\leq \rbr{  1+\frac{1}{H } } \Gamma(h+1)+   \order \rbr{HC
+ \sqrt{ \frac{\iota |S| K\log(K)}{\eta}  } + \frac{H|S|\log(K)}{\eta}}\\
 &\leq  \sum_{h'=h}^{H} \rbr{  1+\frac{1}{H } }^{H-h'}  \order \rbr{HC
  + \sqrt{ \frac{\iota |S| K\log(K)}{\eta}  } + \frac{H|S|\log(K)}{\eta}}     \\
&\leq   \order \rbr{H^2C
  + H \sqrt{ \frac{\iota |S| K\log(K)}{\eta}  } + \frac{H^2|S|\log(K)}{\eta} } , 
\end{align*}
where the second inequality bounds $\eta \leq 1/H$ and the last inequality unrolls from $h+1$ to $H+1$.
\end{proof}

\section{Regret Bound of Epoch V-learning for Adaptive Opponent}
\label{app:reg_bound_evol_adaptive}

\subsection{Regret Bound under $\ENR{}_K$}

\begin{theorem}\label{thm:reg_bound_evol_adaptive}
Suppose that the opponent is adaptive and $K \geq H|S|^{3/2}$. If we run \pref{alg:epoch_V_ol} with $\eta \in \sbr{|S|/K, 1/H}$ and the adversarial bandit subroutine is instantiated by \pref{ex:adv_subroutine} (so $\iota=\Theta(H^2|A|\log(HK|A||S|/\delta))$), then with probability at least $1-\delta$,
\begin{align*}
\ENR_K  \leq \sum_{k=1}^K \rbr{V^k_1(s_1^k)-V^{\mu^k,\nu^k}_1(s_1^k)} \leq \order \rbr{ \eta H^3C
+ H^2|S| \sqrt{\eta \iota  K\log(K) }+ H\sqrt{ \frac{\iota |S|K\log(K)}{\eta}  }   } .
\end{align*}

Further, if we constrain $\eta \in \big[ |S|/K, 1/(H\sqrt{|S|}) \big]$, then $\ENR_K \leq   \order \Big(  \eta H^3C
 + H\sqrt{ \frac{\iota |S|K\log(K)}{\eta}  }  \Big) $.
\end{theorem}

For the adaptive opponent, the regret bound suffers an additional term $\otil(H^2|S| \sqrt{\eta \iota  K })$. Unlike \pref{thm:Bound_final_oblivious}, which only suffers a $\sqrt{|S|}$ dependence, this term has a linear dependence on $|S|$, and thus choosing $\eta \leq 1/H$ cannot suppress it by $\otil \big (H\sqrt{ \frac{\iota |S|K}{\eta}  } \big)$. Compared to the oblivious opponent case, the worse dependence on $|S|$ arises because $V^*_h(s)$ depends on the entire sequence of $(\nu^1_h(s),\ldots,\nu^K_h(s))$, which prevents us from directly applying martingale-based concentration bounds to control $\big| \sum_{k \in \calK_{\tau}(h,s)}  \big( V^*_h(s^k_{h+1})- \mbD_{\mu^k_h,\nu^k_h}[P_hV_{h+1}^*](s) \big) \big|$. To address this issue, we discretize each $V^*_h(s)$, and applying a union bound over all discretized points yields the extra dependence on $|S|$.

The proof of \pref{thm:reg_bound_evol_adaptive} is largely based on that of \pref{thm:reg_bound_evol}, and the difference comes from discretizing $V^*_h(s)$ to handle the adaptive adversary.
Recall from \pref{eq:def_event3} that $\calG^3$ is the only place relying on the assumption of oblivious adversary.
As $V^*_h(s)$ depends on entire sequence of opponent's polices, if the opponent selects policies adaptively, $V^*_h(s)$ is not predictable given the history. 
As a result, one cannot simply follow \pref{lem:nice_event_epochvol_oblivious} to apply Azuma-Hoeffding's inequality.
One way to handling the future-dependent issue of $V^*_h(s)$ is to discretize $V^*_h(s)$, and then apply a union bound over all of them.

For any $\Delta>0$ and step $h \in [H+1]$, we define a finite class of functions that $\calV_{H+1,\Delta}=\{f:S_{H+1} \to \{0\}\}$ and for any step $h \in [H]$
\begin{align}
    \calV_{h,\Delta} := \cbr{  f: S_h \to \cbr{ 0,\Delta,\ldots,  \left \lfloor \frac{H-h+1}{\Delta}   \right  \rfloor  \Delta,H-h+1 }  }.
\end{align}

With the definition of $\calV_{h,\Delta} $, we then define event $\widetilde{\calG}^3$ as:
\begin{equation} \label{eq:wtil_calG3}
    \widetilde{\calG}^3 := \cbr{ \forall h,s,\tau : \max_{f \in \calV_{h+1,\Delta} } \abr{ \sum_{t \in \calK_{\tau}(h,s)} \epsilon_h^t \rbr{s;f}  } \leq \frac{1}{2} \sqrt{|S|}N_{\tau}(h,s)  \beta_{N_{\tau}(h,s)}  }, 
\end{equation}

\begin{definition} \label{def:high_prob_calE_adaptive}
Let $\widetilde{\calG}$ be the event such that $\calG=\calG^1 \cap \calG^2 \cap \widetilde{\calG}^3 \cap \calG^4$.
\end{definition}
\begin{lemma}\label{lem:nice_event_epochvol_adaptive}
We have $\P(\widetilde{\calG}) \geq 1-\delta$.
\end{lemma}
\begin{proof}
As \pref{lem:nice_event_epochvol_oblivious} has shown $\P(\calG^1 \cap \calG^2  \cap \calG^4) \geq 1-3\delta/4$, it suffices to show $\P(\widetilde{\calG}^3) \geq 1-\delta/4$. The argument for proving $\P(\widetilde{\calG}^3) \geq 1-\delta/4$ is analogous to that of $\calG^2$, and the only difference is to apply an additional union bound over all $f \in \calV_{h+1,\Delta}$.
\end{proof}

Then, we show how to modify lemmas that are affected by the discretization step.

\begin{lemma}[Counterpart of \pref{lem:L_minus_Vstar}] \label{lem:L_minus_Vstar_adaptive}
Suppose that $\widetilde{\calG}$ holds where $\widetilde{\calG}$ is defined in \pref{def:high_prob_calE_adaptive} and $\Delta=1/K$.
For any pair $(h,s)$ and any epoch $\tau$ of this pair, we have
\begin{align*} 
&  \sum_{t \in \calK_{\tau}(h,s)} \rbr{r_h^t + V^{t}_{h+1}(s^t_{h+1}) -V^{*}_h(s) } \\
&\leq    \order \rbr{H C_{\tau}(h,s) + \Delta N_{\tau}(h,s)+ \sqrt{|S|} N_{\tau}(h,s)\beta_{N_{\tau}(h,s)} }   + \sum_{t \in \calK_{\tau}(h,s)} \rbr{ V^{t}_{h+1}(s^t_{h+1}) - V^{*}_{h+1}(s^t_{h+1})  }.
\end{align*}
\end{lemma}

\begin{proof}
We show that for $(h,s)$ pair and any epoch $\tau$
\begin{align*}
& 
 \sum_{t \in \calK_{\tau}(h,s)} \rbr{r_h^t +V^{t}_{h+1}(s^t_{h+1}) - V^{*}_h(s) }   \\
&= 
\sum_{t \in \calK_{\tau}(h,s)} \rbr{r_h^t +V^{*}_{h+1}(s^t_{h+1})-V^{*}_h(s) } + \sum_{t \in \calK_{\tau}(h,s)} \rbr{ V^{t}_{h+1}(s^t_{h+1}) - V^{*}_{h+1}(s^t_{h+1})  }   .
\numberthis{} \label{eq:above_bound_1}
\end{align*}

Then, we have
\begin{align*}
&  \sum_{t \in \calK_{\tau}(h,s)} \rbr{r_h^t +V^{*}_{h+1}(s^t_{h+1})-V^{*}_h(s) } \\
&\leq \sum_{t \in \calK_{\tau}(h,s)} \rbr{\mbD_{\mu^t_h,\nu^t_h} \sbr{ r_h + P_h V^{*}_{h+1}}(s) -V^{*}_h(s)} +2\Delta N_{\tau}(h,s) + \order \rbr{  \sqrt{|S|} N_{\tau}(h,s)\beta_{N_{\tau}(h,s)} } \\
&= \sum_{t \in \calK_{\tau}(h,s)} \rbr{ \mbD_{\mu^t_h,\nu^t_h} \sbr{ r_h + P_h V^{*}_{h+1}}(s) - \mbD_{\mu^*_h,\nu^*_h} \sbr{ r_h + P_h V^{*}_{h+1} }(s)} \\
&\quad + 2\Delta N_{\tau}(h,s) + \order \rbr{  \sqrt{|S|} N_{\tau}(h,s)\beta_{N_{\tau}(h,s)} } \\
&\leq \sum_{t \in \calK_{\tau}(h,s)} \rbr{ \mbD_{\mu^t_h,\nu^t_h} \sbr{ r_h + P_h V^{*}_{h+1}}(s) - \mbD_{\mu^t_h, \nu^*_h } \sbr{ r_h + P_h V^{*}_{h+1} }(s)} \\
&\quad + 2\Delta N_{\tau}(h,s) + \order \rbr{  \sqrt{|S|} N_{\tau}(h,s)\beta_{N_{\tau}(h,s)} } \\
& \leq \order \rbr{ HC_{\tau}(h,s) + \Delta N_{\tau}(h,s)  +\sqrt{|S|} N_{\tau}(h,s)\beta_{N_{\tau}(h,s)}  },
\end{align*}
where the first inequality first rounds each $V^{*}_{h+1}(s^t_{h+1})$ to a function $f \in \calV_{h+1,\Delta}$ such that $\norm{f-V^{*}_{h+1}}_{\infty} \leq \Delta$, then uses \pref{eq:wtil_calG3} of $\widetilde{\calG}$ together with the fact that for any step $h$ and $\Delta =1/K$, we have $\log(|\calV_{h,\Delta}|) \leq \order(|S|\log(H/\Delta))$, and finally rounding each $f$ back to $V^{*}_{h+1}(s^t_{h+1})$. Rounding twice incurs extra $2\Delta N_{\tau}(h,s)$ term; the last inequality uses \pref{eq:bound_epoch_corruption}.

The claimed result thus follows.
\end{proof}

\begin{lemma}[Counterpart of \pref{lem:gamma_bound_oblivious_app_version}] \label{lem:gamma_bound_adaptive}
Suppose that $\widetilde{\calG}$ holds where $\widetilde{\calG}$, $\eta \in (0,1/H]$, and $\Delta=1/K$.
For any step $h \in [H]$, we have
\begin{align*}
&\sum_{s \in S_h} \sum_{\tau \in \calE(h,s)}N_{\tau}(h,s) \rbr{  L_{\tau}(h,s) -V^*_h(s)} \\
&\leq \order \rbr{ H^2 C
+ HK\Delta +  H|S| \sqrt{ \frac{\iota  K\log(K)}{\eta}  } + \frac{H^2|S|\log(K)}{\eta}   } .
\end{align*}
\end{lemma}

\begin{proof}
The proof follows the same argument of \pref{lem:gamma_bound} to use \pref{lem:L_minus_Vstar}. Then, the claimed bound is immediate.
\end{proof}

\begin{proof}[Proof of \pref{thm:reg_bound_evol_adaptive}.] This proof mostly follows the argument of \pref{thm:reg_bound_evol} and diverges from \pref{eq:first_term_bound}. Specifically, we apply \pref{lem:gamma_bound_adaptive} in \pref{eq:first_term_bound} to get
\begin{align*}
&\sum_{s \in S_h} \sum_{\tau \in \calE(h,s)}N_{\tau}(h,s)  \rbr{ L_{\tau-1}(h,s) -L_{\tau}(h,s)   } \\
&\leq  \order \rbr{\eta  H^2 C
+\eta HK\Delta + H |S| \sqrt{\eta \iota  K\log(K)  } +\frac{H |S| \log(K)}{\eta} }\\
&\leq  \order \rbr{\eta  H^2 C
  + H |S| \sqrt{\eta \iota  K\log(K)  } +\frac{H |S| \log(K)}{\eta} } , \numberthis{} \label{eq:first_term_bound_adaptive}
\end{align*}
where the second inequality uses $\Delta=1/K$ and $\eta \leq 1/H$ to bound $\eta HK\Delta$ by a constant.

Plugging \pref{eq:first_term_bound_adaptive} and \pref{eq:second_term_bound} into \pref{eq:bound_one_two_terms}, we have that for any step $h \in [H]$
\begin{align} \label{eq:bound_term_onehk_adaptive}
   \sum_{k=1}^K (1)_h^k \leq \order \rbr{ \eta H^2 C
+ H|S|\sqrt{\eta \iota K\log(K) } +\sqrt{ \frac{ \iota |S|K\log(K)}{\eta}  }+\frac{ H |S| \log(K)}{\eta}  }.
\end{align}

\textbf{Putting together.} Plugging \pref{eq:bound_term_onehk_adaptive} and \pref{eq:bound_term_twohk} into \pref{eq:decompose_delta_oneplustwo}, we have that for any step $h \in [H]$
\begin{align*}
    \sum_{k=1}^K \delta^k_h  &\leq \sum_{k=1}^K \delta^k_{h+1}  + \order \rbr{  \eta H^2 C
+ H|S| \sqrt{\eta \iota K\log(K) } +\sqrt{ \frac{ \iota |S|K\log(K)}{\eta}  }+\frac{ H |S| \log(K)}{\eta}  } .
\end{align*}

Summing over all $h \in [H]$, using the fact $\delta_{H+1}^k=0$ for any $k$, and rearranging, we have 
\begin{align*}
\ENR_K \leq \sum_{k=1}^K \delta_1^k &\leq \order \rbr{ \eta H^3C
+ H^2|S| \sqrt{\eta \iota  K\log(K) } +H\sqrt{ \frac{\iota |S|K\log(K)}{\eta}  } +\frac{H^2|S| \log(K)}{\eta} }\\
& \leq \order \rbr{ \eta H^3C 
+ H\sqrt{ \frac{\iota |S|K\log(K)}{\eta}  } +\frac{H^2|S| \log(K)}{\eta} },
\end{align*}
where the second inequality uses the choice $\eta \leq 1/(H\sqrt{|S|})$ to bound $H^2 |S| \sqrt{\eta \iota  K\log(K) } \leq H\sqrt{ \frac{\iota |S|K\log(K)}{\eta}  }$.

Furthermore, if one assumes $K \geq H|S|^{3/2}$ and constrains $\eta \in \sbr{\frac{|S|}{K} , \frac{1}{H \sqrt{|S|}}}$, then,
\begin{align*}
\ENR_K \leq \sum_{k=1}^K \rbr{V^k_1(s_1^k)-V^{\mu^k,\nu^k}_1(s_1^k)} \leq \order \rbr{ \eta H^3C
+ H\sqrt{ \frac{\iota |S|K\log(K)}{\eta}  }   }.
\end{align*}

The proof is thus complete.
\end{proof}

\subsection{Regret Bound under $\NR{}_K$}
\label{app:reg_bound_evol_NR}

In fact, if one evaluates the regret using $\NR_K$, as in \citep{tian2021online}, then the discretization step can be avoided, and the same analysis as in \pref{thm:Bound_final_oblivious} applies, yielding an identical regret bound.
The result of \pref{alg:epoch_V_ol} under $\NR_K$ is presented in the following.

\begin{theorem} \label{thm:nr_reg_epoch_vol}
Suppose that the opponent is adaptive and $K \geq H|S|$. If we run \pref{alg:epoch_V_ol} with $\eta \in \sbr{|S|/K, 1/H}$ and the adversarial bandit subroutine is instantiated by \pref{ex:adv_subroutine} (so $\iota=\Theta(H^2|A|\log(HK|A||S|/\delta))$), then with probability at least $1-\delta$, $\NR_K \leq \order \Big( \eta H^3C
+ H\sqrt{ \frac{\iota |S|K\log(K)}{\eta}  } \Big)$.
\end{theorem}

\section{Regret Bound of Adaptive Epoch V-learning for Oblivious Opponent}

Before proving our main results, we first introduce some notations.

\subsection{Notations.} 
Given any interval $I \subseteq [K]$, we recursively define the locally empirical state Nash values: for each $(h,s) \in [H] \times S_h$,
\begin{align} 
V_{h,I}^{*} \rbr{s} &= \max_{\mu \in \Delta(A_h)} \min_{\nu \in \cbr{\nu^k_h(s) }_{k \in I } } \mbD_{\mu,\nu} \sbr{r_h +P_hV^*_{h+1,I} }(s), \label{eq:def_Vstar_Interval}
\end{align}
with $V_{H+1,I}^{*}(s)=0$ for all $s \in S_{H+1}$.
Based on this, we further define for any interval $I$
\begin{equation} \label{eq:def_reg_interval}
    \ENR^{(b)}_{\ell}(I)=\sum_{k \in \calT^{(b)}_{\ell}} \rbr{V^*_{1,I} (s_1^k) -V^{\mu^k,\nu^k}_1(s_1^k)}.
\end{equation}

If one chooses $I=[K]$, then $\ENR^{(b)}_{\ell}([K])$ is the local regret in sub-block $\ell$ in block $b$. Note that since each sub-block runs a new instance of epoch V-learning algorithm, we can even directly work on $\ENR^{(b)}_{\ell}(\calT^{(b)}_{\ell})$, a stronger regret metric than $\ENR^{(b)}_{\ell}([K])$, and obtain a regret bound in the same order.
Indeed, from the definition of $V_{h,I}^{*} \rbr{s}$ in \pref{eq:def_Vstar_Interval}, one can see that $V^*_{1,[K]} (s_1^k) \leq V^*_{1,\calT^{(b)}_{\ell}} (s_1^k)$, and thus
\[
\ENR^{(b)}_{\ell}([K])=  \sum_{k \in \calT^{(b)}_{\ell}} \rbr{V^*_{1,[K]}  (s_1^k) -V^{\mu^k,\nu^k}_1(s_1^k)} \leq  \ENR^{(b)}_{\ell}(\calT^{(b)}_{\ell})  .
\]

For each step $h$, we define the minimax policy for the opponent when restricted to play a mixture of empirical polices $\{\nu^k_h(s)\}_{k \in I}$ in the interval $I \subseteq [K]$
\begin{align*}
   \nu^{*}_{h,I}(s) &\in \argmin_{ \nu \in \Conv \rbr{\cbr{\nu^k_h(s) }_{k \in I } }  }   \max_{\mu \in \Delta(A_h)}  \mbD_{\mu,\nu} \sbr{r_h +P_h V^*_{h+1,I} }(s).
\end{align*}

For any contiguous interval $I \subseteq [K]$, we define
\[
\calS(I) := \cbr{ [a,b]: 1\leq a \leq b \leq K, [a,b] \subseteq [K] ,I \subseteq J  }.
\]

Let us define
\begin{itemize}
    \item $B$ as the total number of blocks when $K$ episodes end.
    \item $Z_b$ as the total number of sub-blocks in block $b$ when $K$ episodes end.
    \item $L^{(b)}:=1+ \sum_{\ell=1}^{Z_b}\sum_{k \in \calT^{(b)}_{\ell} \backslash \{K\} } \Ind{ \nu^k \neq \nu^{k+1} }$ be the number of policy switches by the opponent (plus one).
\end{itemize}

As each sub-block runs a new instance of epoch V-learning algorithm, we define another form of non-stationary measure:
\begin{align*}
\wtilC =  \sum_{b=1}^B  \sum_{\ell=1}^{Z_b} \wtilC^{(b)}_{\ell},\quad   \wtilC^{(b)}_{\ell}=  \min_{I \in \calS(\calT^{(b)}_{\ell})} \wtilC^{(b)}_{\ell}(I) ,\quad \wtilC^{(b)}_{\ell}(I):= \sum_{h=1}^H \sum_{k \in \calT^{(b)}_{\ell} } \TV \rbr{\nu^k_h(s_h^k),\nu^{*}_{h,I}(s_h^k)  }.
\end{align*}

Here, $\wtilC$ sums the non-stationarity over all sub-blocks.
As $[K] \in \calS(I)$ for any $I$, we have $\wtilC \leq C$ by the definition.

\subsection{Construction of Nice Event}

Let $c_0 \geq 2$ be some absolute constant.
We define
\begin{align}
\calG_5&:= \cbr{ \forall \text{ contiguous interval } I \subseteq [K] : \abr{ \sum_{k \in I} \rbr{\sum_{h=1}^H r^k_h - V_1^{\mu^k,\nu^k}(s_1^k) }   } \leq \sqrt{ \iota |I| }  } ,  \label{eq:def_event5} \\
\calG_6&:= \left\{  \forall b,\ell : \ENR^{(b)}_{\ell}([K]) \leq \sum_{k \in \calT^{(b)}_{\ell}} \delta^k_1 \leq c_0 \rbr{\eta^{(b)}_{\ell} H^3 \wtilC^{(b)}_{\ell}
+ H\sqrt{ \frac{\iota |S||\calT^{(b)}_{\ell}|\log(K)}{\eta^{(b)}_{\ell}}  } }  \right\} . \label{eq:def_event6}
\end{align}

\begin{definition} \label{def:meta_nice_event_calG}
    Let $\overline{\calG}$ be the nice event $\overline{\calG} = \calG_5 \cap \calG_6$.
\end{definition}

\begin{lemma}
Suppose that $K \geq 16H^2|S|$.
We have $\P(\overline{\calG})\geq 1-\delta$.
\end{lemma}
\begin{proof}
 We first prove $\P(\calG_5)\geq 1-\delta/2$. Consider any fixed contiguous interval $I \subseteq [K]$, containing consecutive rounds. Since $\abr{\rbr{\sum_{h=1}^H r^k_h - V_1^{\mu^k,\nu^k}(s_1^k) }} \leq H$ for any episode $k$, applying Azuma-Hoeffding inequality for the interval $I$ ensures with probability at least $1-\delta/(2K^2)$, 
 \[
 \abr{ \sum_{k \in I} \rbr{\sum_{h=1}^H r^k_h - V_1^{\mu^k,\nu^k}(s_1^k) }   } \leq H\sqrt{ 2|I| \log \rbr{ 4K^2/\delta }  } \leq \sqrt{ \iota |I| }.
 \]
 
As $I \subseteq [K]$ is a set of consecutive rounds, there are at most $K(K+1)/2 \leq K^2$ such intervals. Using a union bound over all possible intervals gives that $\P(\calG_5)\geq 1-\delta/2$.

We then show $\P(\calG_6)\geq 1-\delta/2$. Consider any fixed block $b$ and any fixed sub-block $\ell$ in block $b$.
We fix a contiguous interval $\calT^{(b)}_{\ell} \subseteq [K]$, containing consecutive episodes, and further fix an interval $I \in \calS(\calT^{(b)}_{\ell})$ (i.e., $\calT^{(b)}_{\ell} \subseteq I$).
Since \pref{alg:block_alg} runs a new instance of \pref{alg:epoch_V_ol} with input total episode $K$, confidence $\delta/(2K^6)$, and epoch incremental factor $\eta^{(b)}_{\ell} \in [|S|/K,1/H]$, \pref{thm:reg_bound_evol} ensures that with probability at least $1-\delta/(2K^6)$,
\begin{align*}
   \ENR^{(b)}_{\ell}(I) &=\sum_{k \in \calT^{(b)}_{\ell}} \rbr{V^*_{1,I}  (s_1^k) -V^{\mu^k,\nu^k}_1(s_1^k)} \\
    &\leq \sum_{k \in \calT^{(b)}_{\ell}} \rbr{V^k_1(s_1^k) -V^{u^k,v^k}_1(s_1^k)} \leq c_0 \rbr{\eta^{(b)}_{\ell} H^3 \wtilC^{(b)}_{\ell}(I)
+ H\sqrt{ \frac{\iota |S||\calT^{(b)}_{\ell}|\log(K)}{\eta^{(b)}_{\ell}}  }} ,
\end{align*}
where $c_0 >0$ is some absolute constant, and we further constrain $c_0 \geq 2$. We here highlight that when applying \pref{thm:reg_bound_evol}, one just replaces each $V^*_h(s)$ by $V^*_{h,I}(s)$ to obtain the claimed bound since $\calT^{(b)}_{\ell} \subseteq I$ ensures $V^*_{h,I}(s)$ to preserve the optimism.

Based on the fact that $V^*_{1,I}  (s_1^k) \leq V^*_{1,I'} (s_1^k)$ for $I' \subseteq I$, and $[K] \in \calS(\calT^{(b)}_{\ell})$, we have
\begin{equation}
    \ENR^{(b)}_{\ell}([K]) = \sum_{k \in \calT^{(b)}_{\ell}} \rbr{ V^*_{1,[K]}(s_1^k) -V^{\mu^k,\nu^k}_1(s_1^k)} =  \min_{I \in \calS(\calT^{(b)}_{\ell})} \ENR(I).
\end{equation}

Since $|\calS(\calT^{(b)}_{\ell})| \leq K^2$ for any $\calT^{(b)}_{\ell} \subseteq [K]$, using a union bound over all intervals in $\calS(\calT^{(b)}_{\ell})$, we have that with probability at least $1-\delta(2K^4)$
\begin{align*}
    \ENR^{(b)}_{\ell}([K]) &\leq  \min_{I \in \calS(\calT^{(b)}_{\ell})} c_0 \rbr{\eta^{(b)}_{\ell} H^3 \wtilC^{(b)}_{\ell}(I)
+ H\sqrt{ \frac{\iota |S||\calT^{(b)}_{\ell}|\log(K)}{\eta^{(b)}_{\ell}}  }} \\
&=   c_0 \rbr{\eta^{(b)}_{\ell} H^3 \wtilC^{(b)}_{\ell} 
+ H\sqrt{ \frac{\iota |S||\calT^{(b)}_{\ell}|\log(K)}{\eta^{(b)}_{\ell}}  }} .
\end{align*}

As the total number of such interval $\calT^{(b)}_{\ell}$ is also bounded by $K^2$, and the number of blocks and sub-blocks are bounded by $K$, applying a union bound over all $b,\ell,T^{(b)}_{\ell}$ yields $\P(\calG_6)\geq 1-\delta/2$.

Finally, the claimed result follows by using a union bound over $\calG_5,\calG_6$.

\end{proof}

\subsection{Supporting Lemmas}

\begin{lemma} \label{lem:reg_bound_eachblock}
Suppose that $K \geq 16H^2|S|$ and $\overline{\calG}$ holds where $\overline{\calG}$ is given in \pref{def:meta_nice_event_calG}.
There exists an absolute constant $c_0  \geq 2$ given in \pref{eq:def_event6} such that for any block $b$, and any sub-block $\ell$ in block $b$ enjoys the following:
\begin{align*}
    \Phi^{(b)}_{\ell} \leq 2c_0 \rbr{  \eta^{(b)}_{\ell} H^3\wtilC^{(b)}_{\ell}
+ H\sqrt{ \frac{\iota |S| |\calT^{(b)}_{\ell}| \log(K)}{\eta^{(b)}_{\ell}   }  } }.
\end{align*}
\end{lemma}

\begin{proof}
For any block $b$ and any sub-block $\ell$ in this block, one can show
\begin{align*}
    \Phi^{(b)}_{\ell} & =\sum_{k \in \calT^{(b)}_{\ell} } \rbr{V^k_1(s_1^k) -\sum_{h=1}^H r_h^k} +\sqrt{\iota \abr{\calT^{(b)}_{\ell}}  }\\
    &\leq \sum_{k \in \calT^{(b)}_{\ell}} \rbr{V^k_1(s_1^k) -V^{u^k,v^k}_1(s_1^k)}+ 2\sqrt{\iota |\calT^{(b)}_{\ell}|} \\
    &\leq c_0 \rbr{ \eta^{(b)}_{\ell}  H^3 \wtilC^{(b)}_{\ell}
  +H \sqrt{ \frac{\iota |S| |\calT^{(b)}_{\ell}| \log(K)}{ \eta^{(b)}_{\ell}   }  }   }+ 2\sqrt{\iota |\calT^{(b)}_{\ell}|} \\
    &\leq 2c_0\rbr{ \eta^{(b)}_{\ell}  H^3 \wtilC^{(b)}_{\ell}
  +H \sqrt{ \frac{\iota |S| |\calT^{(b)}_{\ell}| \log(K)}{ \eta^{(b)}_{\ell}   }  }   } ,
\end{align*}
where the first inequality uses $\calG_5$ of $\overline{\calG}$ and the definition of $\iota$, and the third inequality follows from $\calG_6$ of $\overline{\calG}$.
\end{proof}

\begin{lemma} \label{lem:num_blocks}
Suppose that $\overline{\calG}$ holds and $K \geq 16H^2|S|$.
We have
\begin{align*}
 B\leq   1+  \log_4^+ \rbr{  4\rbr{ \frac{H(1+\wtilC)^2}{ \iota |S|K\log(K)} }^{1/3}   },
\end{align*}
where $\log_4^+(x)=\max\{\log_4(x),0\}$. Moreover, $\eta^{(b)}_{2^{2b}+1}=\frac{2^{-2b}}{H}$ for all $b \leq B$.
\end{lemma}

\begin{proof}
If $B=1$, then the claimed bound on $B$ holds trivially. Then, we consider the case $B>1$.
It suffices to show that there exists a block such that the last sub-block termination condition will not be met.
From \pref{lem:reg_bound_eachblock}, there exists an absolute constant $c_0 \geq 2$ such that for any block $b$, the last sub-block $\ell=2^{2b}+1$ enjoys the following
\[
\Phi^{(b)}_{\ell}    \leq 2c_0 \rbr{  \eta^{(b)}_{\ell} H^3 (1+\wtilC)
 +H \sqrt{ \frac{\iota |S|K\log(K)}{ \eta^{(b)}_{\ell} }  }   },
\]
where the inequality simply bounds $|\calT^{(b)}_{\ell}| \leq K$ and $\wtilC^{(b)}_{\ell} \leq 1+\wtilC$.

For shorthand, we use $\Lambda_b=\eta^{(b)}_{2^{2b}+1}$ to denote the 
 epoch incremental factor for the last sub-block of block $b$.
Since $\Lambda_b$ is non-increasing w.r.t. $b$ and $\wtilC \leq HK$, there exist at least one blocks $b \leq \log_4 \big(\frac{K}{|S|H} \big)$ such that $\Lambda_{b}  H^3(1+\wtilC) \leq H \sqrt{ \frac{\iota |S|K\log(K)}{ \Lambda_{b}  }  } $. 
Then, the existence can be verified by using the assumption $K \geq 16H^2 |S|$ to show that for $b=\lfloor \log_4 \big(\frac{K}{|S|H} \big) \rfloor$, we have
\[
\Lambda_{b}  H^3(1+\wtilC) =  \frac{H^2(1+\wtilC)}{2^{2b}}  \leq  \frac{2KH^3}{2^{2b}} \leq  8|S|H^4 \leq \frac{1}{2} KH^2 \leq H \sqrt{ \frac{\iota |S|K\log(K)}{ \Lambda_{b}  }  },
\]
where the first equality follows from the fact that $\Lambda_b=\frac{2^{-2b}}{H}$ since $b < \log_4(K/(|S|H))$ implies $2^{-2b}/H \geq |S|/K$, the first inequality bounds $1+\wtilC \leq 2KH$, the second inequality bounds $b=\lfloor \log_4 \big(\frac{K}{|S|H} \big) \rfloor \geq \log_4 \big(\frac{K}{|S|H} \big)-1$, the third inequality uses the assumption $K \geq 16H^2|S| $, and the last inequality follows from facts that $\iota \geq H^2$ and $\Lambda_b \leq H^{-1}4^{-\log_4\big(\frac{K}{|S|H} \big)+1}=4|S|/K$.

Furthermore, $B>1$ also implies that there exist at least one blocks $b \leq \log_4 \big(\frac{K}{|S|H} \big)$ such that $\Lambda_{b} H^3(1+\wtilC) > H \sqrt{ \frac{\iota |S|K\log(K)}{   \Lambda_{b}   }  }$.
Thus, there should exist a block $\hat{b} \leq \log_4 \big(\frac{K}{|S|H} \big)$ such that
\begin{equation} \label{eq:two_sided_ineq}
    \Lambda_{\hat{b}}  H^3(1+\wtilC) \leq H \sqrt{ \frac{\iota |S|K\log(K)}{ \Lambda_{\hat{b}}  }  } \quad \text{and}
\quad  \Lambda_{\hat{b}-1} H^3(1+\wtilC) > H \sqrt{ \frac{\iota |S|K\log(K)}{   \Lambda_{\hat{b}-1}   }  }.
\end{equation}

In such a block $\hat{b}$, the last sub-block $\ell=2^{2\hat{b}}+1$ satisfies
\[
\Phi^{(\hat{b})}_{\ell} \leq 2c_0 \rbr{  \Lambda_{\hat{b}} H^3 \wtilC
+ H\sqrt{ \frac{\iota |S|K\log(K)}{  \Lambda_{\hat{b}}   }  } }  \leq  4c_0 H \sqrt{ \frac{\iota |S|K\log(K)}{  \Lambda_{\hat{b}}  }  }  ,
\]
where the first inequality uses \pref{lem:reg_bound_eachblock}, and the second inequality uses $ \Lambda_{\hat{b}}  H^3(1+\wtilC) \leq H \sqrt{ \frac{\iota |S|K\log(K)}{ \Lambda_{\hat{b}}  }  } $  given in \pref{eq:two_sided_ineq}.
Thus, the block termination condition will never be met in block $\hat{b}$.
Moreover, $\hat{b} \leq \log_4 \big(\frac{K}{|S|H} \big)$ implies that for all $b \leq \hat{b}$, $\Lambda_b= \frac{2^{-2b}}{H} \geq  \frac{|S|}{K}$.
Notice that $\Lambda_{\hat{b}-1} H^3(1+\wtilC) > H \sqrt{ \frac{\iota |S|K\log(K)}{ \Lambda_{\hat{b}-1}   }  } $ gives
\begin{align*}
    \frac{2^{-2(\hat{b}-1)}}{H} =\Lambda_{\hat{b}-1} \geq \rbr{  \frac{\iota |S|K\log(K)}{(1+\wtilC)^2H^4}  }^{1/3}  \Longrightarrow  \hat{b} \leq \log_4 \rbr{  4\rbr{ \frac{H(1+\wtilC)^2}{ \iota |S|K\log(K)} }^{1/3}   }. 
\end{align*}

Combining two cases of $B=1$ and $B>1$, we obtain the claimed bound on $B$.

Finally, as $\Lambda_b=  \frac{2^{-2b}}{H}$ for all $b \leq \log_4( \frac{K}{|S|H})$ and we know the upper bound of $B$, it suffices to show that $\rbr{ \frac{4^3H(1+\wtilC)^2}{ \iota |S|K\log(K)} }^{1/3} \leq K/(|S|H)$ to conclude the proof. One can easily verify that
\begin{align*}
 \frac{4^3 H(1+\wtilC)^2}{ \iota |S|K\log(K)}    \leq \frac{4^4 H^3 K}{\iota|S|} \leq \frac{4^4 H K}{|S|} = \frac{4^4 H K^3}{|S|K^2}  \leq  \frac{ K^3}{|S|^3H^3} ,
\end{align*}
where the first inequality bounds $(1+\wtilC)^2 \leq (2HK)^2$ and $\log(K) \geq 1$, the second inequality bounds $\iota \geq H^2$, and the last inequality uses the assumption that $K \geq 16H^2|S|$.

The proof is thus complete.
\end{proof}

\begin{lemma} \label{lem:fixed_opponent_policy}
Suppose that $\overline{\calG}$ holds and $K \geq 16H^2|S|$.
For any block $b$ and any sub-block $\ell \leq 2^{2b}$ of block $b$, if $v^k=v^{k+1}$ for all $k \in \calT^{(b)}_{\ell}$, then
\begin{align*}
\Phi^{(b)}_{\ell} \leq 2c_0 \sqrt{H^3 \iota |S| \abr{\calT^{(b)}_{\ell}} \log(K)} ,
\end{align*}
where $c_0  \geq 2$ is an absolute constant given in \pref{eq:def_event6}.
\end{lemma}
\begin{proof}
Recall that \pref{alg:block_alg} runs a new instance of \pref{alg:epoch_V_ol} in each sub-block. If the opponent uses a fixed policy in a sub-block, then, $\wtilC^{(b)}_{\ell}=0$. Hence, \pref{lem:reg_bound_eachblock} gives that $\Phi^{(b)}_{\ell} \leq 2c_0  H\sqrt{ \frac{\iota |S||\calT^{(b)}_{\ell}|\log(K)}{\eta^{(b)}_{\ell}}  } =2c_0 \sqrt{H^3 \iota |S| \abr{\calT^{(b)}_{\ell}} \log(K)}$ where the equality holds since for any block $b$, $\eta^{(b)}_{\ell}=H^{-1}$ for all sub-blocks $\ell \leq 2^{2b}$. 
\end{proof}

\begin{lemma} \label{lem:regret_bound_phase_one}
Suppose that $\overline{\calG}$ holds and $K \geq 16H^2|S|$.
For each block $b$, we have
\[
\sum_{\ell=1}^{Z_b} \Phi_{\ell}^{(b)} \leq \order \big( \sqrt{ \min\{2^{2b},L^{(b)}\} H^3 \iota |S| K \log(K) } \big).
\]
\end{lemma}
\begin{proof}
Consider any block $b$. 
We then consider the following two cases.

\textbf{Case 1: $Z_b  \leq 2^{2b}$.}
In this case, when a sub-block $\ell$ in block $b$ ends, we have $\Phi_{\ell}^{(b)} \leq \order \big(H \sqrt{ \frac{  \iota |S| |\calT_{\ell}| \log(K)}{   \eta^{(b)}_{\ell} } } \big)$.
Since $Z_b  \leq 2^{2b}$ implies that $\eta^{(b)}_{\ell}=1/H$ for all $\ell \leq Z_b  \leq 2^{2b}$, we have
\begin{align*}
    \sum_{\ell=1}^{ Z_b    } \Phi_{\ell}^{(b)} & \leq \order \rbr{  \sum_{\ell=1}^{ Z_b    }  \sqrt{ H^3 \iota |S| |\calT_{\ell}|   \log(K) }  }  \leq \order \rbr{   \sqrt{Z_b H^3 \iota |S| K \log(K) }   } ,
\end{align*}
where the last inequality uses Cauchy–Schwarz inequality and $\sum_{\ell=1}^{ Z_b} |\calT_{\ell}|\leq K$.

\textbf{Case 2: $Z_b  = 2^{2b}+1$.} By the termination condition of sub-block $Z_b=2^{2b}+1$, we have
\begin{align*}
    \Phi_{Z_b}^{(b)} \leq \order \rbr{   H \sqrt{   \frac{ \iota |S|K \log(K)  }{  \eta^{(b)}_{Z_b}  }    }   } =  \order \rbr{  \sqrt{ 2^{2b} H^3    \iota |S|K \log(K)     }   }  \leq  \order \rbr{  \sqrt{ Z_b H^3    \iota |S|K \log(K)     }   } .
\end{align*}
where the equality uses \pref{lem:num_blocks} that $\eta^{(b)}_{Z_b} =2^{-2b}/H$ and the last inequality uses $2^{2b} \leq Z_b$.

Then, one can show that
\begin{align*}
\sum_{\ell=1}^{ Z_b  } \Phi_{\ell}^{(b)} &=\sum_{\ell=1}^{ Z_b -1   } \Phi_{\ell}^{(b)} +\Phi_{Z_b}^{(b)}  \leq  \order \rbr{  \sqrt{ Z_b H^3   \iota |S|K \log(K)     }   },
\end{align*}
where the first summation can be bounded via the same argument used in Case 1.

\textbf{Putting together.} In both cases, we have $    \sum_{\ell=1}^{ Z_b    } \Phi_{\ell}^{(b)}  \leq \order \rbr{   \sqrt{Z_b H^3 \iota |S| K \log(K) }   } $.
If the opponent always uses a fixed policy in a sub-block, then \pref{lem:fixed_opponent_policy} implies that this sub-block never ends. 
In other words, if a sub-block ends, then opponent switches the policy at least once.
Thus, $Z_b \leq L^{(b)}$.
Combining this and the fact that $Z_b \leq \order(2^{2b})$, we obtain the desired bound.
\end{proof}

\begin{proof}[Proof of \pref{thm:Bound_final_oblivious}.]
Assume that $K \geq 16H^2|S|$.
For every block $b$, we have $ \ENR^{(b)}_{\ell}([K]) \leq \Phi_{\ell}^{(b)}$ since 
\begin{align*} 
     \ENR^{(b)}_{\ell}([K]) &\leq \sum_{k \in \calT^{(b)}_{\ell}} \rbr{V^k_1(s_1^k) -V^{\mu^k,\nu^k}_1(s_1^k)} \\
     &\leq \sum_{k \in \calT^{(b)}_{\ell} } \rbr{V^k_1(s_1^k) -\sum_{h=1}^H r_h^k} + \sqrt{\iota \abr{\calT^{(b)}_{\ell}}  } =\Phi^{(b)}_{\ell} , \numberthis{} \label{eq:R_b_ell_smaller_Phi_b_ell}
\end{align*}
where the first inequality uses the optimism of epoch-V-ol, and the second inequality uses $\calG_5$ of $\overline{\calG}$.

Recall that $Z_{b}$ is the total number of sub-blocks in block $b$, and $B$ is the total number of blocks.
The regret can be written as
\begin{align*}
&\ENR_K \\
&= \sum_{b=1}^{B} \sum_{\ell=1}^{Z_b} \ENR^{(b)}_{\ell}([K])\\
 &\leq \sum_{b=1}^{B} \sum_{\ell=1}^{Z_b} \Phi^{(b)}_{\ell} \\
&\leq \order \rbr{ \sum_{b=1}^{B} \sqrt{ \min \cbr{ 2^{2b},L^{(b)} } H^3 \iota |S| K \log(K) }}   \\
& \leq \order \rbr{ \min \cbr{\sum_{b=1}^{B} 2^{b}, \sum_{b=1}^{B} \sqrt{L^{(b)}} }\sqrt{H^3\iota |S|K\log(K)  } } \\
& \leq \order \rbr{ \min \cbr{2^{B}, \sum_{b=1}^{B} \sqrt{L^{(b)}} }\sqrt{H^3\iota |S|K\log(K)  } } \\
& \leq \order \rbr{ \min \cbr{2^{B}, \sqrt{(L+\log(K) \log(K) } }\sqrt{H^3\iota |S|K\log(K)  } } \\
& \leq \order \rbr{   \min \cbr{\sqrt{H^3\iota |S|K\log(K)  } + \rbr{ H^5 \iota |S| CK \log (K) }^{\frac{1}{3}}, \sqrt{   (L+\log(K)H^3 \iota |S|K \log^2 (K) } }    },
\end{align*}
where the first inequality uses \pref{eq:R_b_ell_smaller_Phi_b_ell} to bound $\sum_{b=1}^{B} \sum_{\ell=1}^{Z_b}  \ENR^{(b)}_{\ell}([K])  \leq \sum_{b=1}^{B} \sum_{\ell=1}^{Z_b} \Phi^{(b)}_{\ell}$, the second inequality follows from \pref{lem:regret_bound_phase_one}, the fifth inequality uses Cauchy–Schwarz inequality and the fact that $B \leq \order(\log(K))$ and $\sum_{b=1}^B L^{(b)} \leq \order(B+L) \leq \order(\log(K)+L)$, and the last inequality uses \pref{lem:num_blocks} to bound $B$ together with the fact that $\wtilC \leq C$.

Finally, if $K \leq 16H^2|S|$, one can bound regret by $\order(H^2|S|)$.
The claimed bound thus follows.
\end{proof}

\section{Regret Bound of Adaptive Epoch V-learning for Adaptive Opponent}
\label{app:Bound_final_adaptive}

\begin{theorem} \label{thm:Bound_final_adaptive}
For adaptive adversary, running \pref{alg:block_alg} by choosing $\eta^{(b)}_{\ell}$ as: $\eta^{(b)}_{\ell}=1/(H\sqrt{|S|})$ for all $\ell \leq 2^{2b}$ and $\eta^{(b)}_{\ell}=  \max \big \{2^{-2b}/(H \sqrt{|S|}) ,|S|/K\big \}$ otherwise, with probability at least $1-\delta$
\begin{align*}
\ENR_K \leq \otil \rbr{ H^2 |S|^{3/2}+     \min \cbr{|S|^{3/4} \sqrt{H^3\iota  K  }  +  \rbr{ \iota H^5   |S|K  C }^{\frac{1}{3}}, |S|^{3/4} \sqrt{LH^3 \iota  K }}    } .
\end{align*}

Further, running \pref{alg:block_alg} with the choice of $\eta^{(b)}_{\ell}$ specified in \pref{eq:eta_b_ell} guarantees, with probability at least $1-\delta$
\begin{align*}
\NR_K \leq \otil \rbr{ H^2 |S| +     \min \cbr{  \sqrt{H^3\iota |S| K  }  +  \rbr{ \iota H^5   |S|K   C }^{\frac{1}{3}},  \sqrt{LH^3 \iota |S| K }}    } .
\end{align*}
\end{theorem}

\subsection{Supporting Lemmas}

\textbf{Nice event $\overline{\calG}$.} We use the nice event $\overline{\calG}$ defined in \pref{def:meta_nice_event_calG} for the following analysis. For $\calG_6$ given in \pref{eq:def_event6}, one can again show $\P(\calG_6) \geq 1-\delta/2$ by using \pref{thm:Bound_final_adaptive}.

\begin{lemma} \label{lem:num_blocks_adaptive}
Suppose that $\overline{\calG}$ holds and $K \geq 16H^2|S|^{3/2}$.
We have
\begin{align*}
 B\leq   1+  \log_4^+ \rbr{  4\rbr{ \frac{H(1+\wtilC)^2}{ \iota |S|^{5/2}K\log(K)} }^{1/3}   },
\end{align*}
where $\log_4^+(x)=\max\{\log_4(x),0\}$. Moreover, $\eta^{(b)}_{2^{2b}+1}=\frac{2^{-2b}}{H\sqrt{|S|}}$ for all $b \leq B$.
\end{lemma}

\begin{proof}
This proof follows a similar argument of \pref{lem:num_blocks} with minor modifications for a different schedule of $\eta^{(b)}_{2^{2b}+1}$ and a different constraint $K \geq 16 H^2|S|^{3/2}$.
If $B=1$, then the claimed bound on $B$ holds trivially. Then, we consider the case $B>1$.
It suffices to show that there exists a block such that the last sub-block termination condition will not be met.
From \pref{lem:reg_bound_eachblock}, there exists an absolute constant $c_0 \geq 2$ such that for any block $b$, the last sub-block $\ell=2^{2b}+1$ enjoys the following
\[
\Phi^{(b)}_{\ell}    \leq 2c_0 \rbr{  \eta^{(b)}_{\ell} H^3 (1+\wtilC)
 +H \sqrt{ \frac{\iota |S|K\log(K)}{ \eta^{(b)}_{\ell} }  }   }.
\]

For shorthand, we use $\Lambda_b=\eta^{(b)}_{2^{2b}+1}$ to denote the 
 epoch incremental factor for the last sub-block of block $b$.
Since $\Lambda_b$ is non-increasing w.r.t. $b$ and $\wtilC \leq HK$, there exist at least one blocks $b \leq \log_4 \big(\frac{K}{|S|^{3/2}H} \big)$ such that $\Lambda_{b}  H^3(1+\wtilC) \leq H \sqrt{ \frac{\iota |S|K\log(K)}{ \Lambda_{b}  }  } $. 
Then, the existence can be verified by using the assumption $K \geq 16H^2 |S|^{3/2}$ to show that for $b=\lfloor \log_4 \big(\frac{K}{|S|^{3/2}H} \big) \rfloor$, we have
\[
\Lambda_{b}  H^3(1+\wtilC) =  \frac{H^2(1+\wtilC)}{2^{2b} \sqrt{|S|} }  \leq  \frac{2KH^3}{2^{2b} \sqrt{|S|} } \leq  8|S| H^4 \leq \frac{1}{2} KH^2 \leq H \sqrt{ \frac{\iota |S|K\log(K)}{ \Lambda_{b}  }  },
\]
where the first equality follows from the fact that $\Lambda_b=\frac{2^{-2b}}{H\sqrt{|S|}}$ since $b < \log_4(K/(|S|^{3/2}H))$ implies $2^{-2b}/(H\sqrt{|S|}) \geq |S|/K$, the first inequality bounds $1+\wtilC \leq 2KH$, the second inequality bounds $b=\lfloor \log_4 \big(\frac{K}{|S|^{3/2}H} \big) \rfloor \geq \log_4 \big(\frac{K}{|S|^{3/2} H} \big)-1$, the third inequality uses the assumption $K \geq 16H^2|S|^{3/2} $, and the last inequality follows from facts that $\iota \geq H^2$ and $\Lambda_b \leq  \frac{1}{H\sqrt{|S|}} 4^{-\log_4\big(\frac{K}{|S|^{3/2}H} \big)+1}=4|S|/K$.

Furthermore, $B>1$ also implies that there exist at least one blocks $b \leq \log_4 \big(\frac{K}{|S|^{3/2}H} \big)$ such that $\Lambda_{b} H^3(1+\wtilC) > H \sqrt{ \frac{\iota |S|K\log(K)}{   \Lambda_{b}   }  }$.
Thus, there should exist a block $\hat{b} \leq \log_4 \big(\frac{K}{|S|^{3/2}H} \big)$ such that
\begin{equation} \label{eq:two_sided_ineq}
    \Lambda_{\hat{b}}  H^3(1+\wtilC) \leq H \sqrt{ \frac{\iota |S|K\log(K)}{ \Lambda_{\hat{b}}  }  } \quad \text{and}
\quad  \Lambda_{\hat{b}-1} H^3(1+\wtilC) > H \sqrt{ \frac{\iota |S|K\log(K)}{   \Lambda_{\hat{b}-1}   }  }.
\end{equation}

In such a block $\hat{b}$, the last sub-block $\ell=2^{2\hat{b}}+1$ satisfies
\[
\Phi^{(\hat{b})}_{\ell} \leq 2c_0 \rbr{  \Lambda_{\hat{b}} H^3 \wtilC
+ H\sqrt{ \frac{\iota |S|K\log(K)}{  \Lambda_{\hat{b}}   }  } }  \leq  4c_0 H \sqrt{ \frac{\iota |S|K\log(K)}{  \Lambda_{\hat{b}}  }  }  ,
\]
where the first inequality uses \pref{lem:reg_bound_eachblock}, and the second inequality uses $ \Lambda_{\hat{b}}  H^3(1+\wtilC) \leq H \sqrt{ \frac{\iota |S|K\log(K)}{ \Lambda_{\hat{b}}  }  } $  given in \pref{eq:two_sided_ineq}.
Thus, the block termination condition will never be met in block $\hat{b}$.
Moreover, $\hat{b} \leq \log_4 \big(\frac{K}{|S|^{3/2}H} \big)$ implies that for all $b \leq \hat{b}$, $\Lambda_b= \frac{2^{-2b}}{H\sqrt{|S|}} \geq  \frac{|S|}{K}$.
Notice that $\Lambda_{\hat{b}-1} H^3(1+\wtilC) > H \sqrt{ \frac{\iota |S|K\log(K)}{ \Lambda_{\hat{b}-1}   }  } $ gives
\begin{align*}
    \frac{2^{-2(\hat{b}-1)}}{H \sqrt{|S|} } =\Lambda_{\hat{b}-1} \geq \rbr{  \frac{\iota |S|K\log(K)}{(1+\wtilC)^2H^4}  }^{1/3}  \Longrightarrow  \hat{b} \leq \log_4 \rbr{  4\rbr{ \frac{H(1+\wtilC)^2}{ \iota |S|^{5/2}K\log(K)} }^{1/3}   }. 
\end{align*}

Combining two cases of $B=1$ and $B>1$, we obtain the claimed bound on $B$.

Finally, as $\Lambda_b=  \frac{2^{-2b}}{H \sqrt{|S|} }$ for all $b \leq \log_4( \frac{K}{|S|^{3/2}H})$ and we know the upper bound of $B$, it suffices to show that $\rbr{ \frac{4^3H(1+\wtilC)^2}{ \iota |S|^{5/2} K\log(K)} }^{1/3} \leq K/(|S|^{3/2}H)$ to conclude the proof. One can easily verify that
\begin{align*}
 \frac{4^3 H(1+\wtilC)^2}{ \iota |S|^{5/2} K\log(K)}    \leq \frac{4^4 H^3 K}{\iota|S|^{5/2}} \leq \frac{4^4 H K}{|S|^{5/2}} = \frac{4^4 H K^3}{|S|^{5/2} K^2}  \leq  \frac{ K^3}{|S|^{27/8}H^3} ,
\end{align*}
where the first inequality bounds $(1+\wtilC)^2 \leq (2HK)^2$ and $\log(K) \geq 1$, the second inequality bounds $\iota \geq H^2$, and the last inequality uses the assumption that $K \geq 16H^2|S|^{3/2}$.

The proof is thus complete.
\end{proof}

\begin{lemma}
Suppose that $\overline{\calG}$ holds and $K \geq 16H^2|S|^{3/2}$.
For any block $b$ and any sub-block $\ell \leq 2^{2b}$ of block $b$, if $\nu^k=\nu^{k+1}$ for all $k \in \calT^{(b)}_{\ell}$, then
\begin{align*}
\Phi^{(b)}_{\ell} \leq 2c_0   |S|^{3/4} \sqrt{H^3 \iota  \abr{\calT^{(b)}_{\ell}} \log(K)} ,
\end{align*}
where $c_0  \geq 2$ is an absolute constant given in \pref{eq:def_event6}.
\end{lemma}
\begin{proof}
This proof follows from a similar argument of \pref{lem:fixed_opponent_policy} with a different choice of $\eta^{(b)}_{\ell}$.
\end{proof}

\begin{lemma} \label{lem:regret_bound_phase_one_adaptive}
Suppose that $\overline{\calG}$ holds and $K \geq 16H^2|S|^{3/2}$. 
For each block $b$, we have
\[
\sum_{\ell=1}^{Z_b} \Phi_{\ell}^{(b)} \leq \order \big( |S|^{3/4} \sqrt{ \min\{2^{2b},L^{(b)}\} H^3 \iota  K \log(K) } \big).
\]
\end{lemma}
\begin{proof}
This proof follows from a similar argument of \pref{lem:regret_bound_phase_one} with a different choice of $\eta^{(b)}_{\ell}$.
\end{proof}

\subsection{Proof of \pref{thm:Bound_final_adaptive}}

Assume that $K \geq 16H^2|S|^{3/2}$.
Recall that $Z_{b}$ is the total number of sub-blocks in block $b$.
The regret can be written as
\begin{align*}
&\ENR_K \\
&= \sum_{b=1}^{B} \sum_{\ell=1}^{Z_b} R^{(b)}_{\ell}([K])\\
 &\leq \sum_{b=1}^{B} \sum_{\ell=1}^{Z_b} \Phi^{(b)}_{\ell} \\
&\leq \order \rbr{ \sum_{b=1}^{B} |S|^{3/4} \sqrt{ \min \cbr{ 2^{2b},L^{(b)} } H^3 \iota  K \log(K) }}   \\
& \leq \order \rbr{ \min \cbr{\sum_{b=1}^{B} 2^{b}, \sum_{b=1}^{B} \sqrt{L^{(b)}} }  |S|^{3/4} \sqrt{H^3\iota  K\log(K)  } } \\
& \leq \order \rbr{ \min \cbr{2^{B}, \sum_{b=1}^{B} \sqrt{L^{(b)}} }|S|^{3/4} \sqrt{H^3\iota K\log(K)  } } \\
& \leq \order \rbr{ \min \cbr{2^{B}, \sqrt{(L+ \log(K) ) \log(K) } }|S|^{3/4} \sqrt{H^3\iota  K\log(K)  } } \\
& \leq \order \rbr{   \min \cbr{|S|^{\frac{3}{4}  } \sqrt{H^3\iota  K\log(K)  } + \rbr{ H^5 \iota |S| CK \log (K) }^{\frac{1}{3}}, |S|^{\frac{3}{4}} \sqrt{   (L+\log(K))H^3 \iota K \log^2 (K) } }    },
\end{align*}
where the first inequality uses \pref{eq:R_b_ell_smaller_Phi_b_ell} to bound $\sum_{b=1}^{B} \sum_{\ell=1}^{Z_b}  \ENR^{(b)}_{\ell}([K])  \leq \sum_{b=1}^{B} \sum_{\ell=1}^{Z_b} \Phi^{(b)}_{\ell}$, the second inequality follows from \pref{lem:regret_bound_phase_one_adaptive}, the fifth inequality uses Cauchy–Schwarz inequality and the fact that $B \leq \order(\log(K))$ and $\sum_{b=1}^B L^{(b)} \leq \order(B+L) \leq \order(\log(K)+L)$, and the last inequality uses \pref{lem:num_blocks_adaptive} to bound $B$ together with the fact that $\wtilC \leq C$.

Finally, if $K \leq 16H^2|S|^{3/2}$, one can bound regret by $\order(H^2|S|^{3/2})$.
The claimed bound thus follows.

\end{document}